\documentclass[lettersize,journal]{IEEEtran}
\usepackage{amsmath,amsfonts}
\usepackage{algorithmic}
\usepackage{algorithm}
\usepackage{array}
\usepackage{subfigure}
\usepackage{textcomp}
\usepackage{stfloats}
\usepackage{url}
\usepackage{verbatim}
\usepackage{graphicx}
\usepackage{cite}
\usepackage{bm}
\usepackage{booktabs}
\usepackage{xcolor}
\usepackage{epstopdf}
\usepackage{amsthm}
\usepackage[inline, shortlabels]{enumitem}
\theoremstyle{plain}
\newtheorem{theorem}{Theorem}[section]
\newtheorem{proposition}[theorem]{Proposition}

\theoremstyle{definition}

\theoremstyle{remark}

\newcommand{\bE}{\mathbb{E}}

\newcommand{\vmu}{{\boldsymbol \mu}}
\newcommand{\vepsilon}{{\boldsymbol \epsilon}}
\newcommand{\vphi}{{\boldsymbol \phi}}
\newcommand{\vtheta}{{\boldsymbol \theta}}
\newcommand{\vSigma}{{\boldsymbol \Sigma}}

\newcommand{\hD}{\mathcal{D}}
\newcommand{\hE}{\mathcal{E}}

\newcommand{\hH}{\mathcal{H}}

\newcommand{\hL}{\mathcal{L}}

\newcommand{\hN}{\mathcal{N}}

\newcommand{\hR}{\mathcal{R}}

\newcommand{\hU}{\mathcal{U}}

\newcommand{\vzero}{{\bf 0}}
\newcommand{\vx}{{\bf x}}
\newcommand{\vz}{{\bf z}}
\newcommand{\vc}{{\bf c}}
\newcommand{\vy}{{\bf y}}
\newcommand{\vI}{{\bf I}}
\newcommand{\vd}{{\bf d}}

\begin{document}

\title{Domain-Guided Conditional Diffusion Model \\
for Unsupervised Domain Adaptation}

\author{Yulong~Zhang,~
Shuhao~Chen,~
Weisen~Jiang,~Yu~Zhang,~
Jiangang~Lu,~ 
and~James T. Kwok
\IEEEcompsocitemizethanks{\IEEEcompsocthanksitem
Yulong Zhang is with the State Key Laboratory of Industrial Control Technology, College of Control Science and Engineering, Zhejiang University, Hangzhou, China. (e-mail: zhangylcse@zju.edu.cn)

Shuhao Chen and Yu Zhang are with the Department of Computer Science and Engineering, Southern University of Science and Technology, Shenzhen, China. (e-mail: 12232388@mail.suctech.edu.cn; yu.zhang.ust@gmail.com)

Weisen Jiang is with the Department of Computer Science and Engineering, Southern University of Science and Technology and Hong Kong
University of Science and Technology. (e-mail:
wjiangar@cse.ust.hk)

Jiangang Lu is with the State Key Laboratory of Industrial Control Technology, College of Control Science and Engineering, Zhejiang University, Hangzhou, China, and also with the Zhejiang Laboratory, Hangzhou, China. (e-mail: lujg@zju.edu.cn)

James T. Kwok is with the Department of Computer Science and Engineering, Hong Kong
University of Science and Technology, Hong Kong, China. (e-mail:
jamesk@cse.ust.hk).

Yulong Zhang and Shuhao Chen contributed equally to this work.

Corresponding author: Yu Zhang.}

}

% \markboth{IEEE Transactions on Neural Networks and Learning Systems}%
{}

\maketitle

\begin{abstract}
Limited transferability hinders the performance of deep learning models when applied to new application scenarios.
Recently, Unsupervised Domain Adaptation (UDA) has achieved significant progress in addressing this issue via learning domain-invariant features.
However, the performance of existing UDA methods is constrained by the large domain shift and limited target domain data.
To alleviate these issues, we propose DomAin-guided Conditional Diffusion Model (DACDM) to generate high-fidelity and diversity samples for the target domain.
In the proposed DACDM, by introducing class information, the labels of generated samples can be controlled, and a domain classifier is further introduced in DACDM to guide the generated samples for the target domain.
The generated samples help existing UDA methods transfer from the source domain to the target domain more easily, thus improving the transfer performance.
Extensive experiments on various benchmarks demonstrate that DACDM brings a large improvement to the performance of existing UDA methods.
\end{abstract}


\section{Introduction}
Deep neural networks have achieved remarkable advances in a variety of applications due to their powerful representation learning capabilities~\cite{li2021survey, zhang2021survey, subramanian2022generalization, zhang2022diversifying}.
However, when domain shift occurs, a model trained from a source domain may suffer
noticeable performance degradation in the target domain \cite{oza2023unsupervised}.
To handle this issue, Unsupervised Domain Adaptation (UDA) \cite{yang2020transfer,
gu2022unsupervised, zhao2020review} is proposed to learn knowledge from the source
domain as well as unlabeled target domain, and then transfer the knowledge to help learn from the unlabeled target domain.
    
To alleviate domain shift, many UDA methods \cite{long2015learning, MDD, zhu2020deep, chen2020domain, rangwani2022closer, zhang2023free} have been proposed to learn domain-invariant features explicitly or implicitly.
Specifically, \cite{long2015learning, MDD, zhu2020deep} explicitly minimize the distribution discrepancy between the source and target domains.
On the other hand, adversarial-based UDA methods enhance transfer learning
capabilities by confusing the source and target domain features \cite{chen2020domain, rangwani2022closer, zhang2023free} or generating auxiliary samples to bridge the source and target domains \cite{hoffman2018cycada, yang2020bi}.
% However, large domain shifts bring challenges to existing UDA methods.
However, the performance of existing UDA methods is constrained by the unlabeled target domain. 
When there is a large domain shift across domains, it further poses challenges to the UDA methods.
Moreover, when the target domain samples are limited\footnote{For example, in the experiments, the Dslr and Webcam domains in the \textit{Office31} dataset have an average of only 16 and 26 samples per class, respectively. The Art domain in the \textit{Office-Home} dataset has an average of 37 samples per class.}, it is difficult for UDA methods to accurately model the data distribution of the target domain. As a result, direct domain alignment between the source samples and limited target samples could cause suboptimal transfer effects.
% However, it is time-consuming and expensive to collect a sufficient number of target domain samples in some practical scenarios \cite{chai2021deep, zhang2023domain}.
% Hence, researchers aim to reduce their reliance on a large amount of target domain data \cite{zhuang2020comprehensive}.
% Due to the limited and unlabeled target samples\footnote{For example, in the experiments, the Dslr and Webcam domains in the \textit{Office31} dataset have an average of only 16 and 26 samples per class, respectively. The Art domain in the \textit{Office-Home} dataset has an average of 37 samples per class.}, it is difficult for UDA methods to accurately model the data distribution of the target domain.
% As a result, direct domain alignment between the source samples and limited target samples could cause suboptimal transfer effects.
% \textcolor{blue}{However, the performance of UDA methods is limited by the existing data, and the target domain is unlabeled, which further poses challenges.}
    
To address the aforementioned limitations, in this paper, we propose a DomAin-guided Conditional Diffusion Model (DACDM) to generate ``labeled" target samples for UDA methods.
Built on the Diffusion Probabilistic Models (DPM) \cite{ho2020denoising, dhariwal2021diffusion}, which are also called diffusion models, 
DACDM can control the class and domain of each generated sample. 
Specifically, to control the class of each generated sample,
we inject label information into the adaptive group normalization (AdaGN) layers \cite{dhariwal2021diffusion}, where for source samples, we use their ground truth labels, and for unlabeled target samples,
we use their pseudo-labels predicted by a pretrained UDA model. 
To control the domain of each generated sample,
we train a domain classifier to guide the diffusion sampling process toward the target domain in the generation,
where the DPM-Solver++ sampler \cite{lu2022dpm} is used for efficiency.
The generated samples for the target domain are combined with the source samples as an augmented source domain. 
Finally, we train a UDA model to transfer from the augmented source domain to the target domain.
Empirical results show that the target domain is closer to the augmented source domain than the original source domain, making it easier for UDA methods to perform transfer learning.

The proposed DACDM framework is a plug-and-play module that can be integrated into any existing UDA method to boost performance.
We theoretically analyze the generalization bound of DACDM.
Empirically, we combine the proposed DACDM framework with some state-of-the-art UDA methods and conduct extensive experiments on four UDA benchmark datasets to demonstrate the superiority of the proposed DACDM framework.

The contributions of this paper are four-fold.
\begin{enumerate}[\indent 1)]
\item We bring domain guidance into DPMs and propose a novel DACDM framework that
directly generates high-fidelity and diverse samples for the target domain. 
To the best of our knowledge, this is the first application of DPMs for UDA.
\item DACDM can be integrated into any existing UDA methods to improve performance. 
Specifically, in our experiments, we combine DACDM with MCC\cite{jin2020minimum} and ELS \cite{zhang2023free} to achieve state-of-the-art performance on various benchmark datasets.
\item We establish a generalization bound of DACDM.
\item We perform extensive experiments on UDA benchmark datasets to demonstrate
the effectiveness of the proposed DACDM method. Results show that the generated
samples are similar to the target samples, and the augmented source domain consisting of source samples and generated target samples can help reduce domain shift.
\end{enumerate}

\section{Related Work} \label{section:RW}

\subsection{Unsupervised Domain Adaptation}
    
UDA methods \cite{long2015learning, ganin2016domain, zhao2020review} extract knowledge from the labeled source domain to facilitate learning in the unlabeled target domain.
As the data distribution of the target domain differs from that of the source
domain, various methods are proposed to reduce the domain discrepancy. They can
mainly be classified into two categories: discrepancy-based methods and adversarial methods.
    
Most discrepancy-based methods learn a feature extractor to minimize the
distribution discrepancy between source and target domains. For example,    DAN
\cite{long2015learning} tries to minimize the maximum mean discrepancy (MMD)
\cite{gretton2012kernel} between domains, while \cite{zhu2020deep} considers the
relationship between two subdomains within the same class but across different domains.
Margin disparity discrepancy (MDD) \cite{MDD} performs domain alignment with 
generalization error
analysis.
Adaptive feature norm (AFN) \cite{AFN} progressively adapts the feature norms of the source and target domains to improve transfer performance.
Maximum structural generation discrepancy (MSGD) \cite{xia2022maximum} introduces an intermediate domain to gradually reduce the domain shift between the source and target domains.
    Unlike the above methods that explicitly align domains, minimum class confusion (MCC) \cite{jin2020minimum} introduces a general class confusion loss as regularizer.

On the other hand, adversarial methods align the source and target distributions through adversarial training \cite{goodfellow2014explaining}.
For example, domain adversarial neural network (DANN) \cite{ganin2016domain}
learns a domain discriminator to distinguish samples in the two domains, and uses a feature generator to confuse the domain discriminator. 
Conditional domain adversarial network (CDAN) \cite{long2018conditional} injects class-specific information into the discriminator to facilitate the alignment of multi-modal distributions. 
Smooth domain adversarial training (SDAT) \cite{rangwani2022closer} employs
sharpness-aware minimization \cite{foret2020sharpness} to seek a flat minimum for better generalization. 
Additionally, environment label smoothing (ELS) \cite{zhang2023free} alleviates the effect of label noise by encouraging the domain discriminator to output soft domain labels.
    
Adversarial training can also serve as a generative method to bridge the source and target domains.
For instance, cycle-consistent adversarial domain adaptation (CyCADA)
\cite{hoffman2018cycada} uses the GAN for image-to-image translation via the cycle consistency loss \cite{zhu2017unpaired} and semantic consistency loss. 
Bi-directional generation (BDG) \cite{yang2020bi} uses two cross-domain generators to synthesize data of each domain conditioned on the other, and learns two task-specific classifiers.

The aforementioned generation-based UDA methods \cite{hoffman2018cycada,
yang2020bi} require complex mechanisms to preserve the detailed and semantic
information in the generated images, leading to training instability and requires
careful tuning of the hyperparameters and regularizers for convergence \cite{dhariwal2021diffusion, murphy2023probabilistic}.
Different from the GAN-based methods, the proposed DACDM framework directly generates target samples.
Instead of using adversarial training strategies as in CyCADA and BDG, DACDM is based on the diffusion models \cite{dhariwal2021diffusion}, which are better at covering the modes of a distribution than GANs \cite{nichol2021improved}. In DACDM, the label and domain information of each generated sample are controllable.

\subsection{Diffusion Probabilistic Model (DPM)}
\label{sec:dpm}
In recent years,
the DPM \cite{sohl2015deep} has achieved great success in various generative tasks.
It includes a forward diffusion process that converts an original sample $\vx_0$
to a Gaussian noise $\vx_T$, and a reverse denoising process that infers the Gaussian noise back to a sample.
Specifically, the forward process gradually injects noise to $\vx_0$ until it
becomes a random noise $\vx_T$,  via the sampling:
$q\left( {{\vx_t}| {{\vx_{t - 1}}}} \right) = \mathcal{N} \left( {\sqrt
{\alpha_t} {\vx_{t - 1}},{(1-\alpha_t)}\bf{I} } \right)$, where
$\mathcal{N}(\vmu,\vSigma)$ is the multivariate normal distribution with mean
$\vmu$ and variance $\vSigma$, and ${\alpha _t}$'s follow a decreasing schedule.
On the other hand, the reverse process denoises $\vx_T$ to $\vx_0$ using a decoder
${{p_\vtheta }\left( {{\vx_{t - 1}}| {{\vx_t}}} \right)}$.  Specifically, a model
${\vepsilon_\vtheta }\left( {{\vx_t},t} \right)$, parameterized by $\vtheta$,
    is trained to 
    predict the noise injected in the forward process by minimizing an approximate loss
\begin{align}
\label{eqn_noise_origin}
{\mathbb{E}_{t \sim \hU(1,T),\vx_0, \vepsilon \sim \mathcal{N}\left( {\vzero,\bf{I}} \right)}} {{\omega _t}{{\left\| {\vepsilon  - {\vepsilon _\vtheta }\left( {{\vx_t},t} \right)} \right\|}^2}} ,
\end{align}
where $\hU(1,T)$ is the uniform distribution on $\{1,\dots,T\}$, 
${\omega _t} = (1- \alpha_t)^2/2\sigma _t^2{\alpha _t}\left( {1 - \bar{\alpha}_t}
\right)$ with
$\bar{\alpha}_t = \prod\nolimits_{s = 0}^t {{\alpha _s}} $,
${\sigma _t} = \frac{{1 - {{\bar \alpha }_{t - 1}}}}{{1 - \bar{\alpha}_t}}{(1 - \alpha _t)}$,
and ${\vx_t} = \sqrt {\bar{\alpha}_t} \vx_0 + \sqrt {1 - \bar{\alpha}_t} \vepsilon$.
During inference, the learned model ${\vepsilon _\vtheta }\left( {{\vx_t},t}
\right)$ generates a sample $\vx_0$ by gradually denoising the random noise $\vx_T\sim\mathcal{N}\left( {\vzero,\bf{I}} \right)$ according to ${\vx_{t - 1}} = \frac{1}{{\sqrt {{\alpha _t}} }}\left( {{\vx_t} - \frac{{1 - {\alpha _t}}}{{\sqrt {1 - \bar{\alpha}_t} }}{{\rm{\vepsilon}}_\vtheta }\left( {{\vx_t},t} \right)} \right) + {\sigma _t}\vz$, where $\vz \sim \mathcal{N}\left( {\vzero,\bf{I}} \right)$.

\begin{figure*}[t]
\centering
\includegraphics[width=6.5in]{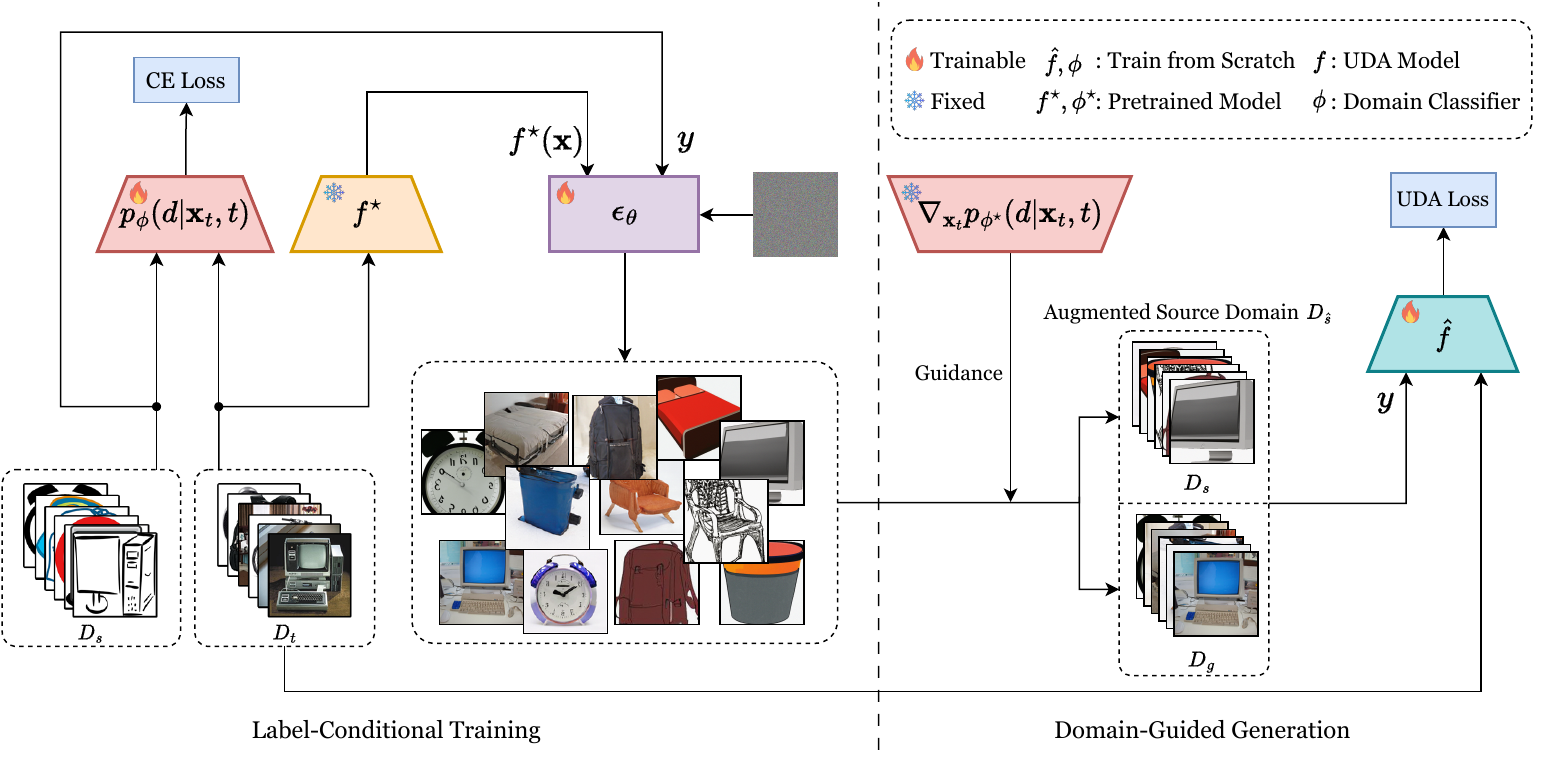}
\caption{An overview of the proposed DACDM framework. In \textit{label-conditional training}, a conditional diffusion model $\vepsilon _\vtheta$ is trained on source samples with ground truth labels and target samples with pseudo labels given by the pretrained classifier $f^{\star}$. 
A domain classifier $\vphi$ is trained to discriminate samples from source and target domains. 
In \textit{domain-guided generation}, the generated target domain data $\mathcal{D}_g$ are sampling from the trained $\vepsilon _\vtheta^{\star}$ with guidance provided by domain classifier $\vphi^\star$. 
Finally, a UDA model $\hat f$ is trained to transfer from the augmented source domain ${\mathcal{D}_{\hat{s}}}=\hD_s\cup \hD_g$ to the target domain $\mathcal{D}_t$.}
\label{model}
\end{figure*}

The DPM is unable to control the classes of generated samples.
To deal with this issue, conditional diffusion probabilistic model (CDPM) \cite{nichol2021improved} incorporates label information into the model by adding a label embedding to the time embedding.
 Dhariwal and Nichol \cite{dhariwal2021diffusion} further improve the CDPM
by embedding 
the condition (i.e., label)
into the AdaGN layers \cite{dhariwal2021diffusion} of the UNet \cite{ronneberger2015u}.
Moreover,
classifier guidance
is also introduced, and
the denoising process of class $c$ is modified as 
$\tilde{\vepsilon}_\vtheta(\vx_t,t,c) = \vepsilon_\vtheta(\vx_t,t,c) - s\sigma_t \nabla_{\vx_t}\log p_\phi(c | \vx_t,t)$,
where $\phi$ is a learned classifier and $s$ is a scale factor.

The DPM suffers from slow sampling, as thousands of denoising steps are required to generate samples \cite{nichol2021improved, ho2020denoising, lu2022dpm, yang2022diffusion}.
To improve efficiency, the denoising diffusion implicit model (DDIM) \cite{ho2020denoising} uses a non-Markov diffusion process to achieve $10\times$ faster sampling.
DPM-Solver \cite{lu2022dpmsolver} formulates DPM as diffusion ordinary
differential equations and reduces the number of denoising steps from 1000 to about 10.
%achieving about $5 \times$ faster than DDIM.
DPM-Solver++ \cite{lu2022dpm} further improves the robustness of classifier guidance in DPM-Solver by using a high-order solver.

\section{Domain-guided Conditional Diffusion Model (DACDM)}
%In this section, we present the proposed DACDM model.

%\subsection{Problem Setting}
In UDA, we are given a labeled source domain ${\hD_s} = \{ (\vx^{sd},y^{sd})\}$
and an unlabeled target domain ${\hD_t} = \{\vx^{td}\}$. Let $N_s=|\hD_s|$ and
$N_t=|\hD_t|$ be the number of samples in $\hD_s $ and $\hD_t$, respectively.
Usually, $N_t$ is much smaller than $N_s$ \cite{yang2020transfer}. 
UDA aims to train a model from $\mathcal{D}_s\cup \hD_t$, and then use this model to make
predictions on $\hD_t$. 
% However, the unlabeled and limited number of target samples makes it difficult to model the entire target domain data distribution.
However, the target domain has limited samples, which are also unlabeled, hindering the transfer performance.

To alleviate this problem, the proposed DACDM aims to generate more ``labeled" target domain samples.
As illustrated in Fig.~\ref{model}, it has two stages: \textit{label-conditioned training} and \textit{domain-guided generation}.
In label-conditioned training, we train a label-conditioned diffusion model on the source and target domain samples to control the labels of generated samples.
For target samples whose labels are unavailable, we first train a UDA model $f^\star$ to predict their pseudo-labels as conditional information.
In domain-guided generation, 
%generating samples for the target domain is more crucial.
we first train a domain classifier $\vphi^\star$, which is then used in domain-guidance
sampling for generating target samples.  Finally, the generated samples are combined
with the source samples as an augmented source domain to train a UDA model.

%In the following two sections, we introduce these two stages in more detail.

\subsection{Label-Conditioned Training}
\label{sec:condition}
%In this stage, we train a label-conditional diffusion model on the source and target domain samples to control the labels of generated samples.  For the target samples whose labels are not available, we train a UDA model 
% to predict their pseudo-labels as conditional information.

To generate more target domain samples, a straightforward solution is to train a DPM on the target domain. 
Analogous to (\ref{eqn_noise_origin}), 
a noise prediction model ${\vepsilon _\vtheta}\left( {{\vx^{td}_t},t} \right)$ 
can be trained
by minimizing the following loss:
\begin{align}
\label{eqn_noise_td}
{\mathbb{E}_{t \sim \hU(1,T),\vx^{td}_0\sim\hD_t, \vepsilon \sim \mathcal{N}\left( {\vzero,\bf{I}} \right)}} {{\omega _t}{{\left\| {\vepsilon  - {\vepsilon _\vtheta }\left( {{\vx^{td}_t},t} \right)} \right\|}^2}}.
\end{align}
However, this cannot control the samples' labels, and thus the generated samples may not cover all classes.

To alleviate this problem, one can 
use the label embedding of $\vx_0^{td}$ as condition $\vc$
and 
feed this to a CDPM \cite{nichol2021improved} 
${\vepsilon_\vtheta}\left( {{\vx^{td}_t},t, \vc} \right)$
(Section~\ref{sec:dpm}).
%whose embedding is injected into the adaptive group normalization (AdaGN) layers of UNet \cite{ronneberger2015u}.
However, while 
labels for the source domain samples are available, those 
for the target domain samples are not.
To address this issue, we propose to first
train a UDA model $f^{\star}$ by minimizing the
following commonly-used objective \cite{tlbook}:
\begin{align}
\label{eqn1}
%f^{\star} = \arg
\min_{f}\frac{1}{N_s}\sum_{ (\vx, y) \in \hD_s } \ell ( f(\vx), y) 
+ \beta\hR(\mathcal{D}_s,\mathcal{D}_t),
\end{align}
where $\ell(\cdot,\cdot)$ is a supervised loss (e.g., cross-entropy (CE) loss for
classification problems), 
%$|\hD_s|$ is the size of $\hD_s$, 
$\beta>0$ is a regularization parameter, and $\hR(\cdot,\cdot)$ is a
loss 
used to bridge the source and target domains
(e.g., domain discrepancy loss 
in discrepancy-based methods
\cite{long2015learning, MDD, jin2020minimum},
and domain discrimination loss 
in adversarial-based methods
\cite{ganin2016domain, rangwani2022closer, zhang2023free}).
We then use 
$f^{\star}$'s 
predictions 
($\hat{\vy}$)
on the
target domain samples 
as their pseudo-labels.
The loss in (\ref{eqn_noise_td}) is consequently changed to:
\begin{equation}
\label{eqn4}
{\bE_{{t\sim \hU(1,T), \vx_0^{td}\sim \hD_t ,\vepsilon \sim \hN \left(
{\vzero,\bf{I}} \right)}}} \omega_t {\left\| {\vepsilon  - {\vepsilon _\vtheta
}\left( \vx_t^{td}, t, f^{\star}(\vx_0^{td})\right)} \right\|^2}.\!\!\!\!
\end{equation}

Using more training data is always beneficial to DPM training.
It can also improve the 
quality  and diversity of the 
generated samples
\cite{croitoru2023diffusion}. 
Thus, instead of using only the target domain samples 
to train $\vepsilon_\vtheta$
as in 
(\ref{eqn4}),
we use both the source and target domain samples 
to minimize the loss 
\begin{align}    \label{eqn:loss_func_2}
%\!\!\!\!\!\hL(\vtheta) \! =\! 
{\bE_{{t\sim \hU(1,T), \vx_0\sim \hD_s\cup \hD_t,\vepsilon \sim \hN \left( {\vzero,\bf{I}} \right)}}} \omega_t {\left\| {\vepsilon \! - \! {\vepsilon _\vtheta }\left( \vx_t, t, \vc \right)} \right\|^2},\!\!\!
\end{align}
where $\vc$
is the (true) label when 
$\vx_0$ is from the 
source domain, and the
pseudo-label predicted by $f^{\star}$ when
$\vx_0$ is from the
target domain.
The complete label-conditioned training procedure is 
shown in Algorithm \ref{algorithm:training_cdpm}.

\begin{algorithm}
\caption{Label-conditioned training.}
\label{algorithm:training_cdpm}
    \renewcommand{\algorithmicrequire}{\textbf{Input:}} 
    \renewcommand{\algorithmicensure}{\textbf{Output:}}
    \begin{algorithmic}[1]
        \REQUIRE{Source domain ${\hD_s}$, target domain ${\hD_t}$, step size $\gamma$, decreasing sequence $\alpha_1,\dots,\alpha_T$, mini-batch $B$;}
        \STATE Train a UDA model $f^\star$ on $\hD_s\cup \hD_t$;
        \REPEAT
        \STATE Sample a mini-batch $\{ {{{\rm{\vx}}_{0,i}}}\}_{i = 1}^B$ from $\hD_s\cup \hD_t$;
        \FOR{$i=1,\dots,B$}
        \IF{$\vx_{0,i} \in \hD_s$}
            \STATE $\vc_i=y_i$;
        \ELSE
            \STATE $\vc_i=f^{\star}(\vx_{0,i})$;
        \ENDIF
        \ENDFOR
        \STATE $t\sim \hU(1,T)$;
        \STATE $\vepsilon\sim \hN(\vzero, \vI)$;
        \STATE $\bar{\alpha}_t = \prod\nolimits_{j = 1}^t {{\alpha _j}}$, ${\sigma _t} = \frac{{1 - {{\bar \alpha }_{t - 1}}}}{{1 - \bar{\alpha}_t}}{(1-\alpha _t)}$;
        \STATE $\vx_{t,i} = \sqrt{\bar{\alpha}_t}\vx_{0,i} + \sqrt {1 - \bar{\alpha}_t}\vepsilon$; $i=1,\dots,B$;
        \STATE $ {\omega _t} = \frac{(1-\alpha _t)^2}{2\sigma _t^2{\alpha _t}\left( {1 - \bar{\alpha}_t} \right)}$;
        \STATE $\hL_{\text{mini-batch}}(\vtheta) = \frac{\omega_t}{B}\sum_{i=1}^B {\left\| {\vepsilon - {\vepsilon_{\vtheta} }\left( {\vx_{t,i}}, t, \vc_i \right)} \right\|^2}$;
        \STATE $\vtheta \gets \vtheta - \gamma \nabla_{\vtheta} \hL_{\text{mini-batch}}(\vtheta)$;
        \UNTIL{converged.}
        \RETURN $\vtheta^\star$.
    \end{algorithmic}
    \end{algorithm}

\subsection{Domain-Guided Generation}
\label{sec:guidance}
While label-conditioned training
can control the 
generated sample's
class label, it
does not control its domain.
In UDA, as we focus on the performance of the target domain,
it is more crucial to generate samples for the target domain.

To achieve this, we use classifier guidance \cite{dhariwal2021diffusion},
and use a domain classifier to guide generation towards the target domain.
Specifically, we first train a domain classifier
(with parameter $\vphi$)
on the noisy images $\vx_t$'s % $p_\vphi(d | \vx_t,t)$, 
by minimizing the loss
\begin{align}
\label{loss_func_3}
\vphi^{\star} =\arg\min_{\vphi} \sum_{t=1}^T\;\sum_{(\vx_0, \cdot) \in \hD_s \cup \hD_t}   \ell(\vphi(\vx_t,t), d),
\end{align}
where 
${\vx_t} = \sqrt {\bar{\alpha}_t} \vx_0 + \sqrt {1 - \bar{\alpha}_t} \vepsilon$,
$\ell(\cdot,\cdot)$ is the cross-entropy loss, 
$d$ is ground-truth domain label of $\vx$.
After obtaining $\phi^\star$, the gradient  
$\nabla_{\vx_t}\log p_{\vphi^{\star}}(d^{td} | \vx_t,t)$ 
(where $p_{\vphi^{\star}}(d^{td} | \vx_t,t)$ is the predicted domain probability, $d^{td}$ denotes the target domain label)
is used to guide the sampling process towards the target domain.
The domain-guided noise prediction model becomes
\begin{align}\!\!\!\tilde{\vepsilon}_\vtheta(\vx_t,t,\hat{\vc})\! = \!\vepsilon_\vtheta(\vx_t,t,y)\!-\! s \sqrt{1-\bar \alpha_t} \; \nabla_{\vx_t}\log p_{\vphi^{\star}}(d^{td} | \vx_t,t),\!\! \!\!\label{eq:temp-askd}
\end{align}
where 
$s$ is a scale factor,
and 
$\hat{\vc}=[y,d^{td}]$
with $y$ being the label
for class conditional controlling.

To accelerate the sampling procedure, 
we use DPM-Solver++ \cite{lu2022dpm} to generate samples ${\mathcal{D}_g}$
for the target domain.
Given an initial noisy $\vx_T \sim \hN(\vzero, \vI)$ and time steps
$\left\{ {{t_i}} \right\}_{i = 0}^M$, the multi-step second-order solver iterates
\begin{align}
\!\!\!\vx_{t_i}& \!=\! {\sqrt{\alpha _{t_i}}}\left( {e^{ - {b_i}} - 1} \right) \bigg( \! \frac{{{b_i}}}{{2{b_{i - 1}}}}{\vx_\vtheta }\!\left( {{{\vx}_{{t_{i - 2}}}}, {t_{i - 2}}, \hat{\vc}} \right) \nonumber \\
\!\!- &\!\Big(\!{1\! +\!\frac{b_i}{{2{b_{i-1}}}}} \!\Big) {\vx_\vtheta }\left( {{{\vx}_{t_{i - 1}}}, {t_{i - 1}},\hat{\vc}} \right) \!\! \bigg) \! +\! \frac{\sqrt{1-\bar \alpha_{t_i}}}{\sqrt{1-\bar \alpha_{t_{i-1}}}}{{\vx_{t_{i-1}}}},
\label{eqn:DPM_Solver++}
\end{align}
for $i=0,...,M$,
where $\vx_\vtheta(\vx_t, t, \hat{\vc}) \! =\! \frac{\vx_t - \sqrt{1-\bar \alpha_t} \tilde{\vepsilon}_\vtheta(\vx_t,t,\hat{\vc})}{\sqrt{\bar \alpha_t}}$, ${b_i} = {\lambda _{t_i}} - {\lambda _{{t_{i - 1}}}}$, and ${\lambda _t} = \frac{1}{2} \log \left( {\frac{\bar \alpha_t }{1 - \bar \alpha _t}} \right)$.
The whole domain-guided generation procedure is shown in Algorithm
\ref{algorithm:DACDM_sampling}.
    
    \begin{algorithm}[!t]
    \caption{Domain-guided generation.}
    \label{algorithm:DACDM_sampling}
    \renewcommand{\algorithmicrequire}{\textbf{Input:}} 
    \renewcommand{\algorithmicensure}{\textbf{Output:}}

\begin{algorithmic}[1]
\REQUIRE{Initial $\vx_T$, time steps $\{t_i\}^M_{i=0}$, noise prediction model $\bm{\vepsilon_{\vtheta^\star}}$ obtained from Algorithm \ref{algorithm:training_cdpm}, domain classifier model 
% $p_{\vphi^{\star}}$, 
$\vphi^{\star}$, 
target domain label $d^{td}$, and guidance scale $s$, decreasing sequence $\alpha_1,\dots,\alpha_T$;}
\STATE $\vx_{t_0} = \vx_T\sim \hN(\vzero, \vI)$;
\STATE Sample a class $y$;
\STATE $\hat{\vc}=[y,d^{td}]$;
        \FOR{$i = 1,\dots, M$}
        \STATE $\bar{\alpha}_{t_i} = \prod\nolimits_{j = 1}^{t_i} {{\alpha_j}}$;
            \STATE $\lambda_{t_i} =\frac{1}{2} \log \left( {\frac{\bar \alpha_{t_i} }{1 - \bar \alpha _{t_i}}} \right)$;
            \STATE $b_i=\lambda_{t_i} - \lambda_{t_{i-1}}$;
        \STATE Compute guidance $\vd_{t_i} \!=\! \sqrt{1\!-\!\bar \alpha_{t_i}}  \nabla_{\vx_{t_i}} \!\!\log p_{\vphi^{\star}}(d | \vx_{t_i}, t_i)$;\!\!
        \STATE $\tilde{\vepsilon}_\vtheta(\vx_{t_i},{t_i},\hat{\vc})=\vepsilon_{\vtheta^\star}(\vx_{t_i},t_i,y) - s\vd_{t_i}$;
        \IF{$i=1$}
          \STATE \!$\vx_{t_i} = \frac{\sqrt{1-\bar \alpha_{t_i}}}{\sqrt{1-\bar \alpha_{t_{i-1}}}} \vx_{t_{i-1}} + \sqrt{ \bar \alpha_{t_i}} ( 1 \!- e^{-b_i} )  \!\bigg(\!\!\vx_{t_{i-1}} \!- \frac{\sqrt{1-\bar \alpha_{t_i}}}{\sqrt{\bar \alpha_{t_i}}}  \tilde{\vepsilon}_\vtheta(\vx_{t_{i-1}},t_{i-1},\hat{\vc})\!\!\bigg)$;
        \ELSE
          \STATE Compute $\vx_{t_i}$ according to Eq. \eqref{eqn:DPM_Solver++};
        \ENDIF
        \ENDFOR
        \RETURN $(\vx_{t_M}, y)$.
    \end{algorithmic}
    \end{algorithm}
    
To reduce the domain shifts, we combine the generated samples $\mathcal{D}_g$ with the source
domain samples $\mathcal{D}_s$ to form an augmented source domain data
${\mathcal{D}_{\hat{s}}}$. The effectiveness of this combination is verified in
Section \ref{Ablation Studies}.
A UDA model is then used to transfer from the augmented source domain
${\mathcal{D}_{\hat{s}}}$ to the target domain $\mathcal{D}_t$ by minimizing the
following objective 
\begin{align} 
\label{eqn:final_UDA}
\hat f =  \arg
\min_{f} \frac{1}{N_{\hat s}}\sum_{ (\vx,y) \in \hD_{\hat{s}}} \ell (f (\vx), y)
+ \beta\hR(\mathcal{D}_{\hat{s}},\mathcal{D}_t),
\end{align}
where $N_{\hat s}=|{\mathcal{D}_{\hat{s}}}|=N_s+N_g$ is the number of samples in $\mathcal{D}_{\hat{s}}$.
Finally, the learned $\hat{f}$ is evaluated on the target domain
samples.
The whole DACDM algorithm is shown in Algorithm \ref{algorithm:overall_algorithm}.
DACDM can be integrated into any UDA method.
In the experiments, DACDM is combined with MCC \cite{jin2020minimum} and ELS \cite{zhang2023free}.

\begin{algorithm}[!t]
\caption{Domain-guided Conditional Diffusion Model (DACDM).}
\label{algorithm:overall_algorithm}
\renewcommand{\algorithmicrequire}{\textbf{Input:}} 
\renewcommand{\algorithmicensure}{\textbf{Output:}}

\begin{algorithmic}[1]
\REQUIRE{Source domain ${\hD_s}$, target domain ${\hD_t}$, number of generated samples $N_g$;}
\STATE Train a label-conditioned diffusion model $\vepsilon_{\vtheta^\star}$ on $\hD_s\cup \hD_t$ by Algorithm \ref{algorithm:training_cdpm};
\STATE Train a domain classifier $\vphi^{\star}$ on $\hD_s\cup \hD_t$;
\STATE $\hD_g = \emptyset$;
\FOR{$i = 1,\dots, N_g$}
\STATE Generate $(\vx_i,y_i)$ with $\vphi^{\star}$ by Algorithm \ref{algorithm:DACDM_sampling};
\STATE Append $(\vx_i,y_i)$ to $\hD_g$;
\ENDFOR
\STATE $\hD_{\hat s} = \hD_g \cup \hD_s$;
\STATE Train a UDA model $\hat f$ on $\hD_{\hat s} \cup \hD_t$;
\STATE Evaluate the learned $\hat{f}$ on $\hD_t$;
\RETURN Evaluation results.
\end{algorithmic}
\end{algorithm}

\subsection{Generalization Analysis}

In this section, we study the generalization properties of DACDM. 
Let
$\hE_{s}(f,f^\prime) = P_{(\vx, y)\in \hD_{s}}(f(\vx)\neq
f^\prime(\vx))$,
$\hE_{t}(f,f^\prime) = P_{(\vx, y)\in \hD_{t}}(f(\vx)\neq
f^\prime(\vx))$,
and $\hE_{g}(f,f^\prime) = P_{(\vx, y)\in \hD_{g}}(f(\vx)\neq
f^\prime(\vx))$
be the disagreement between models $f$ and $f^\prime$ on data
distribution $\hD_{s}$, $\hD_{t}$, and $\hD_{g}$, respectively.
When $f^\prime$ is an oracle (i.e., $f^\prime(\vx)$ is the ground-truth label $y$ of $\vx$), 
we simply
write
$\hE_{s}(f) = P_{(\vx, y)\in \hD_{s}}(f(\vx)\neq  y)$, $\hE_{t}(f) = P_{(\vx, y)\in \hD_{t}}(f(\vx)\neq  y)$,
and $\hE_{g}(f) = P_{(\vx, y)\in \hD_{g}}(f(\vx)\neq  y)$
(the expected risks of $f$ on 
$\hD_{s}$, $\hD_{t}$, and $\hD_{g}$, respectively).
The corresponding empirical risks
%on the source, target, and generated domains 
are  
$\hat \hE_s(f)=\frac{1}{{{N_s}}}\sum_{(\vx_i,y_i) \in \hD_s}{\mathbb{I}({f(\vx_i) \neq y_i})}$;  $\hat \hE_t(f)=\frac{1}{{{N_t}}}\sum_{(\vx_i,y_i) \in \hD_t}{\mathbb{I}({f(\vx_i)\neq y_i})}$ and $\hat \hE_s(f)=\frac{1}{{\hat{N_s}}}\sum_{(\vx_i,y_i) \in \hD_{\hat{s}}} {\mathbb{I}({f(\vx_i)\neq y_i})}$, where $\mathbb{I}(\cdot)$ is the indicator function.

\begin{proposition}
\label{unpperboundProof}
Let the VC dimension 
of the hypothesis space $\mathcal{H}$ be $V$ for classifier model $\hat{f}$.
Given $f^\star_s$ trained from the source domain and $\hat{f}$ in
(\ref{eqn:final_UDA}), for $\delta > 0$, with the probability $1-2\delta$, the
expected risk of $\hE_t(\hat {f})$ satisfies 
\begin{align}
\label{eqn:target_bound}
\hE_t(\hat {f}) \leq & \eta \bigg( {\hat \hE_s}\left( f^\star_s \right) + \frac{1}{2}{d_{{\rm{{\cal H}}}\Delta {\rm{{\cal H}}}}}(\hD_s,\hD_t) + C\bigg) + \varepsilon(\delta, V, N_{\hat s}) \nonumber \\
& + (1-\eta) \bigg(\hat \hE_g(f^\star_s) + \frac{1}{2} d_{{\rm{{\cal H}}}\Delta {\rm{{\cal H}}}}(\hD_g, \hD_t) + C \bigg),
\end{align}
where $\eta=N_s / N_{\hat s}$, $C=4\sqrt{\frac{2V\log (2N_{\hat s})+\log \frac{2}{\delta}}{N_{\hat s}}}$, $\varepsilon(\delta, V, N_{\hat s})=\sqrt{\frac{1}{2N_{\hat s}}\ln
\frac{2V}{\delta}}$,
and ${d_{\mathcal{H} \Delta\mathcal{H}} }(\hD, \hD')$ is the empirical $\mathcal{H} \Delta
\mathcal{H} $-distance between $\hD$ and $\hD'$
\cite{ben2010theory}.
\end{proposition}

\begin{proof}
Using Theorem 2 in \cite{ben2010theory}, for any
$\hat{f},\tilde f \in \hH$, with probability $1- \delta$, 
\[ {\hE_t}(\hat f) 
\leq {\hE_{\hat{s}}}( {\hat f}) + \frac{1}{2}{d_{{\rm{{\cal H}}}\Delta
{\rm{{\cal H}}}}}(\hD_{\hat s},\hD_t) + C. \]

${d_{\mathcal{H} \Delta\mathcal{H}}}( {\hD_{\hat s},\hD_t} )$, the $\mathcal{H} \Delta \mathcal{H}$-distance between ${\hD_{\hat s}}$ and ${\hD_{t}}$, satisfies
\begin{eqnarray}
\lefteqn{\frac{1}{2}{d_{\mathcal{H} \Delta\mathcal{H}} }( {\hD_{\hat s},\hD_t} )} \nonumber\\
&=& \sup_{\hat f,\tilde f} |\hE_t(\hat f,\tilde f)-\hE_{\hat{s}}(\hat f,\tilde f)| \nonumber \\ 
&= & \sup_{\hat f,\tilde f}(|\eta(\hE_t(\hat f,\tilde f)-\hE_s(\hat f,\tilde f)) \nonumber \\
& &+ (1-\eta)(\hE_t(\hat f,\tilde f)-\hE_g(\hat f,\tilde f))|) \nonumber \\
&\leq  & \frac{1}{2}\eta{d_{\mathcal{H} \Delta\mathcal{H}} }\left( {\hD_s,\!\hD_t} \right) \!+\! \frac{1}{2}(1\!-\!\eta){d_{\mathcal{H} \Delta\mathcal{H}} }\left( {\hD_g,\hD_t} \right). \label{proof_inequality_1}
\end{eqnarray}

\begin{figure*}[!tbh]
\centering
\subfigure[\textit{Office-31}.]{
    \includegraphics[width=0.19\textwidth]{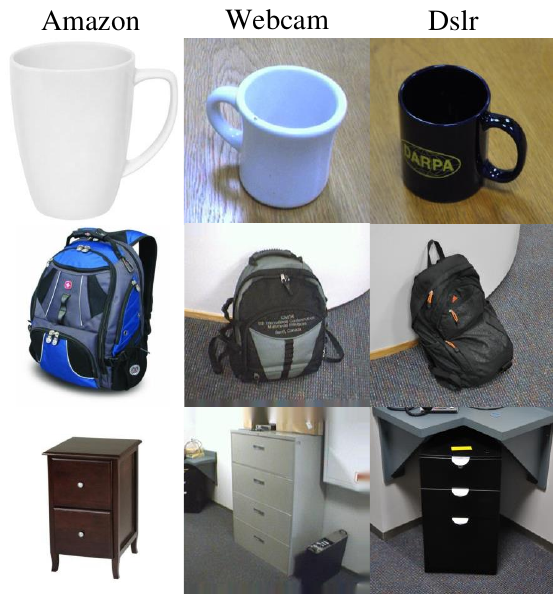}
    \label{fig:example_Office_31}
} 
\subfigure[\textit{Office-Home}.]{
    \includegraphics[width=0.25\textwidth]{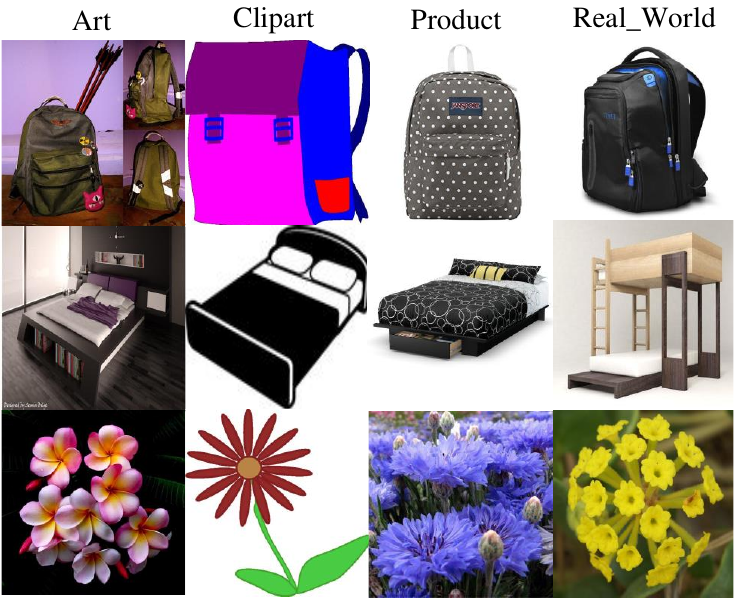}
    \label{fig:example_Office_Home}
} 
\subfigure[\textit{VisDA-2017}.]{
    \includegraphics[width=0.128\textwidth]{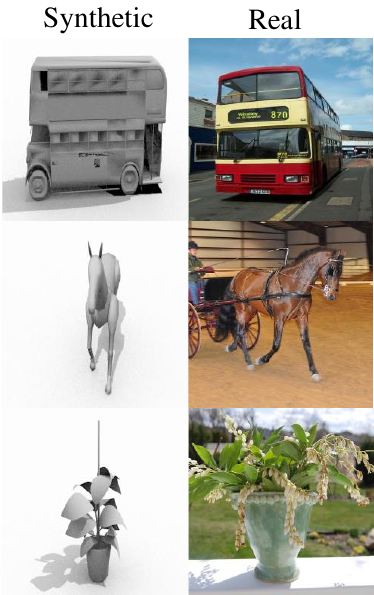}
    \label{fig:example_VisDA}
} 
\subfigure[\textit{miniDomainNet.}]{
    \includegraphics[width=0.25\textwidth]{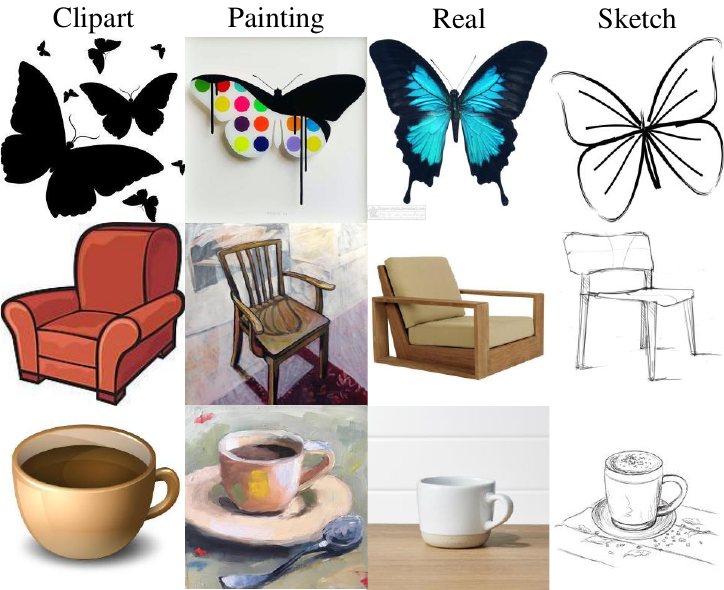}
    \label{fig:example_miniDomainNet}
}
\vskip -.1in
\caption{Example images from \textit{Office-31}, \textit{Office-Home}, \textit{VisDA-2017}, and \textit{miniDomainNet} datasets.}
\label{example}
\end{figure*}

Combining those two inequalities gives
\begin{eqnarray}
\lefteqn{{\hE_t}(\hat f)} \nonumber\\
&\leq& {\hE_{\hat{s}}}( {\hat f} ) + \frac{1}{2}{d_{{\rm{{\cal H}}}\Delta {\rm{{\cal H}}}}}(\hD_{\hat s},\hD_t ) + C \nonumber \\
&=& {\hE_{\hat s}}( {\hat f} ) + {\hat \hE_{\hat s}}\left(f^\star_s \right) - {\hat \hE_{\hat s}}\left( f^\star_s \right) + \frac{1}{2}{d_{{\rm{{\cal H}}}\Delta {\rm{{\cal H}}}}}(\hD_{\hat s},\hD_t ) + C \nonumber \\ 
&\leq & {\hat \hE_{\hat s}}( {\hat f} ) + \varepsilon(\delta, V, N_{\hat s}) + {\hat \hE_{\hat s}}\left(f^\star_s \right) - {\hat \hE_{\hat s}}\left( f^\star_s \right) \nonumber \\
&&+ \frac{1}{2}{d_{{\rm{{\cal H}}}\Delta {\rm{{\cal H}}}}}(\hD_{\hat s},\hD_t ) + C \label{proof_inequality_2} \\ 
&\leq & {\hat \hE_{\hat s}}\left( f^\star_s \right) + \frac{1}{2}{d_{{\rm{{\cal H}}}\Delta {\rm{{\cal H}}}}}(\hD_{\hat s}, \hD_t) + C + \varepsilon(\delta, V, N_{\hat s}) \label{temp:askdlll} \\
&\leq & \eta {\hat \hE_s}\left( f^\star_s \right) + (1-\eta){\hat \hE_g}\left( f^\star_s \right) + \frac{1}{2}\big[ \eta {d_{{\rm{{\cal H}}}\Delta {\rm{{\cal H}}}}}( \hD_{s}, \hD_t ) \nonumber \\ 
&& + \left( {1 - \eta } \right){d_{{\rm{{\cal H}}}\Delta {\rm{{\cal H}}}}}(\hD_g,\hD_t) \big] + C +\varepsilon(\delta, V, N_{\hat s}) \label{temp:askdlllsk} \\ 
&=& \eta \bigg( {\hat \hE_s}\left( f^\star_s \right) + \frac{1}{2}{d_{{\rm{{\cal H}}}\Delta {\rm{{\cal H}}}}}(\hD_s,\hD_t) + C \bigg) + \varepsilon(\delta, V, N_{\hat s}) \nonumber \\
&& + (1-\eta) \bigg(\hat  \hE_g(f^\star_s) + \frac{1}{2} d_{{\rm{{\cal H}}}\Delta {\rm{{\cal H}}}}(\hD_g, \hD_t) + C  \bigg). \nonumber
\end{eqnarray}
Inequality \eqref{proof_inequality_2} holds with probability $1-\delta$ (Eq. (2.1) in \cite{abu2012learning}).
Inequality \eqref{temp:askdlll} holds since $\hat{f}$ is optimized by the UDA methods based on the augmented source domain and target domain, while $f^\star_s$ is optimized only on the source domain, leading to $ {\hat \hE_{\hat s}}( {\hat f} ) \leq {\hat \hE_{\hat{s}}}\left( f^\star_s \right)$, and inequality \eqref{temp:askdlllsk} holds due to ${\hat \hE_{\hat s}}( f^\star_s)=\eta {\hat \hE_s}\left( f^\star_s \right) + (1-\eta){\hat \hE_g}\left( f^\star_s \right)$ and inequality \eqref{proof_inequality_1}.
\end{proof}

Proposition \ref{unpperboundProof} shows that the expected risk of DACDM on the
target domain consists of two terms. The first one is determined by the empirical 
risk on the source domain 
and the distance between the source and target domains,
while the second term depends on the empirical risk on the generated domain and its distance
to the target domain.

\begin{table}[!t]
\centering
\caption{Statistics of the datasets used.}
\vskip -.1in
\begin{tabular}{c c c c c}
    \toprule
    & number of & number of & number of & number of \\
    dataset & \#images & classes & domains & tasks \\
    \midrule
    \textit{Office-31}  & $4,110$ & $31$  & $3$ & $6$ \\
    \textit{Office-Home}  & $15,500$ & $65$ & $4$ &  $12$ \\
    \textit{VisDA-2017}  & $207,785$ & $1$2 & $2$ & $1$ \\
    \textit{miniDomainNet}  & $140,006$ & $126$  & $4$ & $12$ \\
    \bottomrule
\end{tabular}
\label{table:ds} 
\end{table}

\begin{table}[!t]\small
\centering
\caption{Accuracy (\%) on the \textit{Office-31} dataset with \textit{ResNet-50}. 
$\uparrow$ denotes the accuracy improvement brought by the DACDM framework over the corresponding baseline. The best is in \textbf{bold}. 
Results with $^*$ are from the original papers.
}
    \vskip -.1in
    \label{office31}
    \setlength{\tabcolsep}{1.5mm}{
    \resizebox{\columnwidth}{!}{
    \begin{tabular}{ccccccc @{\hskip 0.2in} c}
    \toprule  & A$\rightarrow$W        & D$\rightarrow$W        & W$\rightarrow$D        & A$\rightarrow$D        & D$\rightarrow$A        & W$\rightarrow$A        & Average \\
    \midrule
    ERM \cite{vapnik1999nature}  & 77.07 & 96.60 & 99.20 & 81.08 & 64.11 & 64.01 & 80.35 \\
    DANN \cite{ganin2016domain} & 89.85 & 97.95 & 99.90 & 83.26 & 73.28 & 73.75 & 86.33 \\
    AFN \cite{AFN}  & 91.82 & 98.77 & \textbf{100.00} & 95.12 & 72.43 & 70.71 & 88.14 \\
    CDAN \cite{long2018conditional} & 92.42 & 98.62 & \textbf{100.00}  & 91.44 & 74.61 & 72.80 & 88.32 \\
    BDG$^*$ \cite{yang2020bi}  & 93.60 & 99.00 & \textbf{100.00}  & 93.60 & 73.20 & 72.00 & 88.50 \\
    MDD \cite{MDD}  & 93.55 & 98.66 & \textbf{100.00}  & 93.92 & 75.29 & 73.95 & 89.23 \\
    SDAT \cite{rangwani2022closer} & 91.32 & 98.83 & \textbf{100.00} & 95.25 & 76.97 & 73.19 & 89.26 \\
    MSGD$^*$ \cite{xia2022maximum} & 95.50 & 99.20 & \textbf{100.00} & 95.60 & 77.30 & 77.00 & 90.80 \\
    \midrule
    MCC \cite{jin2020minimum}  & 94.09 & 98.32 & 99.67  & 94.25 & 75.89 & 75.46 & 89.61 \\
    MCC+DACDM  & 95.51 & 98.58 & 99.93  & 95.31 & 78.26 & \textbf{78.43} & 91.01 \\
    \midrule
    ELS \cite{zhang2023free}  & 93.84 & 98.78 & \textbf{100.00} & 95.78 & 77.72 & 75.13 & 90.21 \\
    ELS+DACDM  & \textbf{96.90} & \textbf{98.91} & \textbf{100.00} & \textbf{97.46} & \textbf{79.79} & 77.74 & \textbf{91.80} \\
    \bottomrule
\end{tabular}}}
\end{table}

\begin{table*}[!tbph]\small
\centering
    \caption{Accuracy (\%) on the \textit{Office-Home} dataset. $\uparrow$ denotes the accuracy improvement brought by the DACDM framework over the corresponding baseline (i.e., MCC or ELS). The best is \textbf{in bold}. Results with $^*$ are from the original papers.}
    \vskip -.1in
    \label{officehome}
    \setlength{\tabcolsep}{1mm}{
    \begin{tabular}{ccccccccccccc @{\hskip 0.05in} c}
    \toprule  & Ar$\rightarrow$Cl   & Ar$\rightarrow$Pr  & Ar$\rightarrow$Rw & Cl$\rightarrow$Ar & Cl$\rightarrow$Pr & Cl$\rightarrow$Rw & Pr$\rightarrow$Ar & Pr$\rightarrow$Cl & Pr$\rightarrow$Rw & Rw$\rightarrow$Ar & Rw$\rightarrow$Cl & Rw$\rightarrow$Pr & Average \\
    \midrule
    ERM \cite{vapnik1999nature}  & 44.06 & 67.12 & 74.26 & 53.26 & 61.96 & 64.54 & 51.91 & 38.90 & 72.94 & 64.51 & 43.84 & 75.39 & 59.39 \\
    DANN \cite{ganin2016domain} & 52.53 & 62.57 & 73.20 & 56.89 & 67.02 & 68.34 & 58.37 & 54.14 & 78.31 & 70.78 & 60.76 & 80.57 & 65.29   \\
    AFN \cite{AFN} & 52.58 & 72.42 & 76.96 & 64.90 & 71.14 & 72.91 & 64.08 & 51.29 & 77.83 & 72.21 & 57.46 & 82.09 & 67.99 \\
    CDAN \cite{long2018conditional} & 54.21 & 72.18 & 78.29 & 61.97 & 71.43 & 72.39 & 62.96 & 55.68 & 80.68 & 74.71 & 61.22 & 83.68 & 69.12 \\
    BDG$^*$ \cite{yang2020bi} & 51.50 & 73.40 & 78.70 & 65.30 & 71.50 & 73.70 & 65.10 & 49.70 & 81.10 & 74.60 & 55.10 & 84.80 & 68.70 \\
    MDD \cite{MDD} & 56.37 & 75.53 & 79.17 & 62.95 & 73.21 & 73.55 & 62.56 & 54.86 & 79.49 & 73.84 & 61.45 & 84.06 & 69.75 \\
    SDAT \cite{rangwani2022closer} & 58.20 & 77.46 & 81.35 & 66.06 & 76.45 & 76.41 & 63.70 & 56.69 & 82.49 & 76.02 & 62.09 & 85.24 & 71.85 \\
    MSGD$^*$ \cite{xia2022maximum} & 58.70 & 76.90 & 78.90 & 70.10 & 76.20 & 76.60 & \textbf{69.00} & 57.20 & 82.30 & 74.90 & 62.70 & 84.50 & 72.40 \\
    \midrule
    MCC \cite{jin2020minimum}  & 56.83 & 79.81 & 82.66 & 67.80 & 77.02 & 77.82 & 66.98 & 55.43 & 81.79 & 73.95 & 61.41 & 85.44 & 72.24 \\
    MCC+DACDM & 58.23 & \textbf{80.33} & \textbf{82.91} & \textbf{70.14} & 79.15 & \textbf{81.36} & 68.49 & 57.75 & \textbf{83.44} & 74.18 & 63.81 & \textbf{85.67} & \textbf{73.71}\\
    \midrule
    ELS \cite{zhang2023free}  & 57.79 & 77.65 & 81.62 & 66.59 & 76.74 & 76.43 & 62.69 & 56.69 & 82.12 &  75.63 & 62.85 & 85.35 & 71.84  \\
    ELS+DACDM      & \textbf{60.35} & 78.81 & 82.74 & 69.59 & \textbf{80.53} & 79.55 & 65.16 & \textbf{58.26} & 83.11 & \textbf{75.81} & \textbf{64.18} & 85.55 & 73.64\\
    \bottomrule
\end{tabular}}
\end{table*}

\section{Experiments}
\label{sec:expt}

In this section, we empirically evaluate the proposed DACDM on a number of domain
adaptation datasets.

\subsection{Settings}

    \noindent
    \textbf{Datasets.}
Experiments are performed on four image classification benchmark datasets
(example images are shown in Fig.~\ref{example} and dataset statistics are in Table \ref{table:ds}):
(i) \textit{Office-31} \cite{saenko2010adapting}, 
which contains $4,110$ images from $31$ classes of three domains (Amazon (A), DSLR (D), and Webcam (W)).
Six transfer tasks (A$\rightarrow$W, D$\rightarrow$W, W$\rightarrow$D,  A$\rightarrow$D, D$\rightarrow$A, W$\rightarrow$A) are constructed.
(ii)
\textit{Office-Home} \cite{venkateswara2017deep}, 
which contains $15,500$ images from $65$ classes of four domains (Art (Ar), Clipart (Cl), Product (Pr), and Real-World (Rw)). All combinations of domain transfer are considered, leading to a total of $12$ transfer tasks.
(iii) \textit{VisDA-2017} \cite{peng2017visda}, which is a large-scale
synthetic-to-real dataset with $207,785$ images from $12$ classes of two
domains (Synthetic and Real).
Following \cite{jin2020minimum, MDD, rangwani2022closer}, we consider the transfer
task Synthetic $\to$ Real.
(iv) \textit{miniDomainNet} \cite{zhou2021domain},
which is a subset of DomainNet \cite{peng2019moment}, with $140,006$ images from
$126$ classes of four domains (Clipart (C), Painting (P), Real (R), Sketch (S)).
Following \cite{xie2023dirichletbased},
$12$ transfer tasks are constructed.

\noindent{\bf Baselines.} 
We compare with ERM \cite{vapnik1999nature} and a number of UDA methods, including (i) discrepancy-based methods: AFN \cite{AFN}, MDD \cite{MDD}, MSGD \cite{xia2022maximum}, MCC \cite{jin2020minimum}, and (ii) adversarial-based methods:
DANN \cite{ganin2016domain}, CDAN \cite{long2018conditional}, BDG \cite{yang2020bi}, SDAT \cite{rangwani2022closer}, ELS \cite{zhang2023free}.
We combine the proposed DACDM with 
the state-of-the-art in these two categories
(namely, MCC and ELS).
Note that 
the same UDA method 
is used
in \textit{label-conditioned training} and \textit{domain-guided generation},
though in general they can be different.

\noindent\textbf{Implementation Details.} 
To initialize $\vepsilon_\vtheta$,
we use the CDPM in \cite{nichol2021improved}, which is
pretrained on \textit{ImageNet} and has
a U-Net \cite{ronneberger2015u} architecture.
The resolutions
of generated images are $128\times128$ for the \textit{miniDomainNet} dataset, and $256\times256$ for the other three datasets. 
And we use an image size of 96$\times$96 in \textit{miniDomainNet} dataset \cite{zhou2021domain}, and $256\times256$ for the other three datasets \cite{jin2020minimum, zhang2023free}.
The images in the source and target domains are processed to the same resolution as the generated images.
For each class, $200$, $200$, $2,000$, and $200$ images are generated for \textit{Office-31}, \textit{Office-Home}, \textit{VisDA-2017}, and \textit{miniDomainNet}, respectively.
For the backbones of the UDA models, following \cite{rangwani2022closer,zhou2021domain}, \textit{ResNet-50} \cite{he2016deep} is used for both \textit{Office-31} and \textit{Office-Home} datasets,  while \textit{ResNet-101} and \textit{ResNet-18} are used for the \textit{VisDA-2017} and \textit{miniDomainNet} datasets, respectively.
The learning rate scheduler follows \cite{ganin2016domain}.
For ELS+DACDM, the initial learning rate is $0.002$ for \textit{Office-31} and \textit{VisDA-2017}, and $0.01$ for the other two datasets.
As for MCC+DACDM, 
the initial learning rate is $0.002$ for \textit{VisDA-2017}, and $0.005$ for the other three datasets.
All experiments are run on an NVIDIA V100 GPU.

\subsection{Results}

Table \ref{office31} shows the accuracies on six transfer tasks of \textit{Office-31}.
As can be seen, ELS+DACDM performs the best on average.  Furthermore, MCC+DACDM
outperforms MCC, showing the effectiveness of DACDM.  Specifically, ELS+DACDM
achieves an average accuracy of $91.80\%$, which surpasses ELS by $1.59\%$.
MCC+DACDM achieves an average accuracy of $91.01\%$, showing DACDM brings an
improvement of $1.40\%$ over MCC. On the three challenging tasks (A$\to$W,
W$\to$A, and D$\to$A), both MCC+DACDM and ELS+DACDM perform better than the baselines by a large margin.

Table \ref{officehome} shows the accuracies on the $12$ tasks in \textit{Office-Home}.
As can be seen, MCC+DACDM achieves the highest average accuracy.
Furthermore, ELS+DACDM outperforms ELS, showing the effectiveness of DACDM.

\begin{table*}[!tbph]\small
\centering
\caption{Accuracy (\%) on the \textit{VisDA-2017} dataset using \textit{ResNet-101}. 
$\uparrow$ denotes the accuracy improvement brought by the DACDM framework over the corresponding baseline. The best is \textbf{in bold}. Results with $^*$ are from the original papers.}
\vskip -.1in
\label{visda}
\setlength{\tabcolsep}{1.5mm}{
\begin{tabular}{ccccccccccccc @{\hskip 0.2in} c}
\toprule 
& aero & bicycle & bus & car & horse & knife & motor & person & plant & skate & train & truck & mean \\
\midrule 
ERM \cite{vapnik1999nature}  & 81.71 & 22.46 & 54.08 & \textbf{76.21} & 74.83 & 10.69 & 83.81 & 18.71 & 80.88 & 28.66 & 79.66 & 5.98	& 51.47  \\
DANN \cite{ganin2016domain} & 94.75 & 73.47 & 83.46 & 47.91 & 87.00 & 88.30 & 88.47 & 77.18 & 88.16 & 90.05 & 87.21 & 42.26 & 79.02   \\
AFN \cite{AFN} & 93.13 & 54.76 & 81.03 & 69.74 & 92.36 & 75.88 & 92.11 & 73.83 & 93.16 & 55.55 & \textbf{90.48} & 23.63	& 74.64 \\
CDAN \cite{long2018conditional} & 94.55 & 74.41 & 82.22 & 58.92 & 90.56 & 96.22 & 89.71 & 78.90 & 86.11 & 89.06 & 84.81 & 43.42 & 80.74 \\
MDD \cite{MDD} & 92.68 & 65.26 & 82.29 & 66.78 & 91.68 & 92.09 & \textbf{93.18} & 79.67 & 92.12 & 84.95 & 83.85 & 48.66	& 81.10 \\
SDAT \cite{rangwani2022closer} & 94.51 & 83.56 & 74.28 & 65.78 & 93.00 & 95.83 & 89.61 & 80.04 & 90.86 & 91.47 &  84.95 & 54.93 & 83.23  \\
MSGD$^*$ \cite{xia2022maximum} & \textbf{97.50} & 83.40 & \textbf{84.40} & 69.40 & \textbf{95.90} & 94.10 & 90.90 & 75.50 & \textbf{95.50} & \textbf{94.60} & 88.10 & 44.90 & 84.60 \\
\midrule
MCC \cite{jin2020minimum}  & 95.26 & 86.14 & 77.12 & 69.98 & 92.83 & 94.84 & 86.52 & 77.78 & 90.26 & 90.98 & 85.68 & 52.52 & 83.32 \\
MCC+DACDM   & 96.43  & \textbf{87.28} & 83.16 & 74.37 & 94.59 & 96.13 & 88.60 & 81.90 & 92.98 & 94.45 & 87.26 & \textbf{61.56} & \textbf{86.56} \\
\midrule
ELS \cite{zhang2023free} & 94.76 & 83.38 & 75.44 & 66.45 & 93.16 & 95.14 & 89.09 & 80.13 & 90.77 & 91.06 & 84.09 & 57.36 & 83.40 \\
ELS+DACDM   & 96.20  & 84.79 & 83.15 & 73.28 & 94.76 & \textbf{96.58} & 90.99 & \textbf{82.21} & 92.98 & 93.37 & 87.49 & 59.70 & 86.29 \\
\bottomrule
\end{tabular}}
\end{table*}

\begin{table*}[!tbph]\small
\centering
\caption{Accuracy (\%) on the \textit{miniDomainNet} dataset using \textit{ResNet-18}. $\uparrow$ denotes the accuracy improvement brought by the DACDM framework over the corresponding baseline. The best is \textbf{in bold}.}
\vskip -.1in
\label{miniDomainNet}
\setlength{\tabcolsep}{1.5mm}{
\begin{tabular}{ccccccccccccc @{\hskip 0.2in} c}
\toprule  & C$\rightarrow$P   & C$\rightarrow$R  & C$\rightarrow$S & P$\rightarrow$C & P$\rightarrow$R & P$\rightarrow$S & R$\rightarrow$C & R$\rightarrow$P & R$\rightarrow$S & S$\rightarrow$C & S$\rightarrow$P & S$\rightarrow$R & Average \\
\midrule
ERM \cite{vapnik1999nature}  & 39.48 & 53.27 & 42.93 & 49.55 & 68.18 & 41.99 & 49.30 & 55.52 & 37.35 & 54.60 & 45.33 & 53.08 & 49.22 \\
DANN \cite{ganin2016domain} & 45.94 & 56.12 & 49.40 & 50.72 & 65.61 & 50.07 & 55.15 & 60.55 & 49.95 & 58.54 & 54.64 & 58.99 & 54.64 \\
AFN \cite{AFN} & 49.23 & 60.11 & 51.11 & 55.60 & 70.59 & 51.78 & 55.84 & 60.41 & 47.46 & 60.69 & 56.36 & 62.28 & 56.79 \\
CDAN \cite{long2018conditional} & 47.99 & 58.50 & 51.17 & 56.36 & 68.71 & 53.01 & 61.15 & 62.85 & 53.44 & 60.89 & 55.90 & 60.88 & 57.57 \\
MDD \cite{MDD} & 48.53 & 61.75 & 52.32 & 59.74 & 70.62 & 55.43 & 62.18 & 62.22 & 54.04 & 63.07 & 58.55 & 64.50 & 59.41 \\
SDAT \cite{rangwani2022closer} & 50.97 & 62.42 & 53.91 & 60.57 & 69.97 & 55.85 & 64.39 & 64.83 & 55.86 & 64.07 & 59.43 & 64.28 & 60.55 \\
\midrule
MCC \cite{jin2020minimum}  & 51.95 & 67.73 & 52.66 & 60.96 & 76.61 & 54.67 & 64.15 & 64.02 & 50.34 & 63.64 & 59.68 & 69.78 & 61.35 \\
MCC+DACDM & 55.28 & \textbf{70.21} & 54.42 & 63.25 & \textbf{77.05} & 55.74 & \textbf{65.16} & \textbf{65.01} & 51.96 & 65.76 & 60.83 & \textbf{71.09} & 62.98 \\
\midrule
ELS \cite{zhang2023free}  & 50.11 & 61.45 & 53.02 & 60.77 & 70.61 & 56.04 & 62.43 & 64.16 & 54.89 &  63.93 & 59.19 & 64.47 & 60.09 \\
ELS+DACDM      & \textbf{56.14} & 68.52 & \textbf{55.51} & \textbf{64.96} & 73.44 & \textbf{58.23} & 64.67 & 64.36 & \textbf{57.66} & \textbf{67.23} & \textbf{62.56} & 69.90 & \textbf{63.60} \\
\bottomrule
\end{tabular}}
\end{table*}

\begin{figure*}[!tbph]
\centering
\subfigure[Real source samples.]{
    \includegraphics[width=0.18\textwidth]{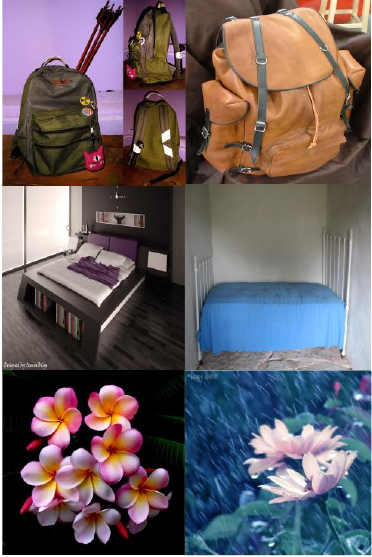}
    \label{fig:Office_home_1}
} \vrule
\subfigure[Real target samples.]{
    \includegraphics[width=0.18\textwidth]{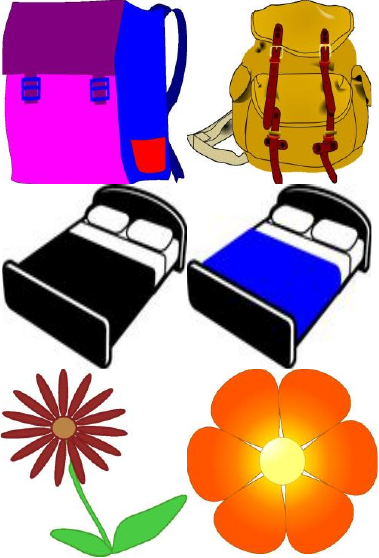}
    \label{fig:Office_home_2}
} \vrule
\subfigure[Generated target samples.]{
    \includegraphics[width=0.53\textwidth]{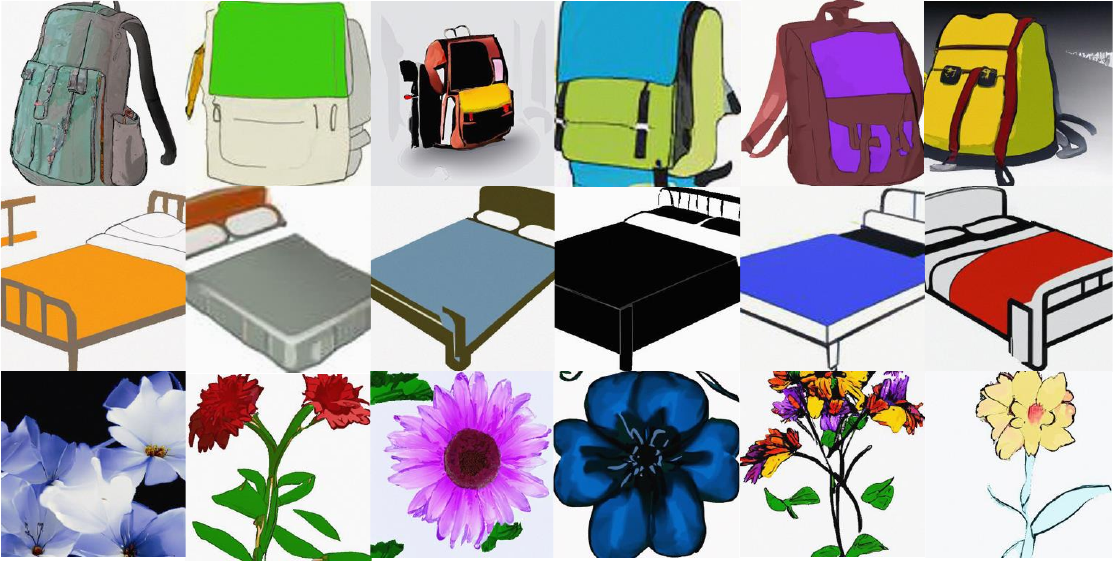}
    \label{fig:Office_home_3}
}
\caption{Real and generated images for the transfer task Ar$\rightarrow$Cl on the \textit{Office-Home} dataset.}
\label{generate1}
\end{figure*}

Table \ref{visda} shows the accuracies on \textit{VisDA-2017}.
As can be seen, 
the proposed DACDM is beneficial to both MCC and ELS. Specifically, DACDM brings noticeable improvements ($>+1\%$) to both MCC and ELS in all classes, demonstrating its effectiveness.

Table \ref{miniDomainNet} shows the accuracy on $12$ tasks in \textit{miniDomainNet}.
Again,
DACDM brings a large improvement to both MCC and ELS in average accuracy.
Moreover, DACDM is beneficial to MCC and ELS on all tasks, demonstrating its effectiveness.

\subsection{Analysis on Generated Samples}
\label{Analyses on Generated Results}

\begin{figure*}[t]
\centering
\subfigure[Real source samples.]{
    \includegraphics[width=0.176\textwidth]{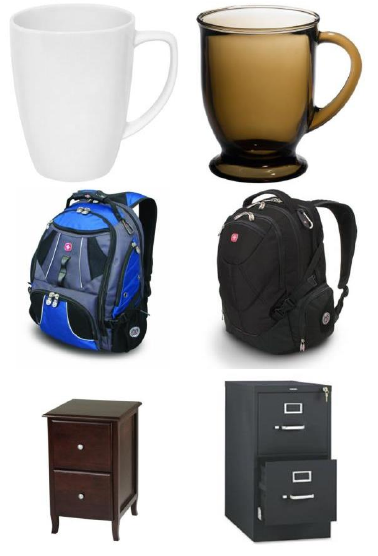}
    \label{fig:Office31_1}
} \vrule
\subfigure[Real target samples]{
    \includegraphics[width=0.176\textwidth]{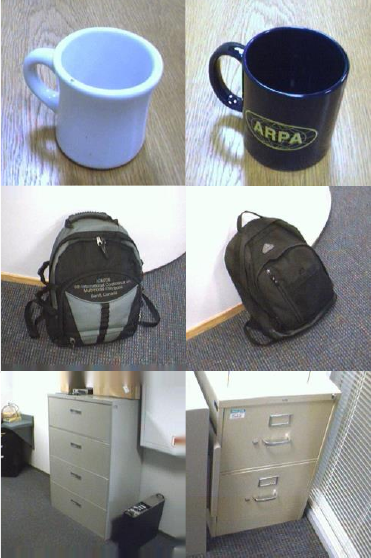}
    \label{fig:Office31_2}
} \hspace{-2mm}  \vrule
\subfigure[Generated target samples.]{
    \includegraphics[width=0.53\textwidth]{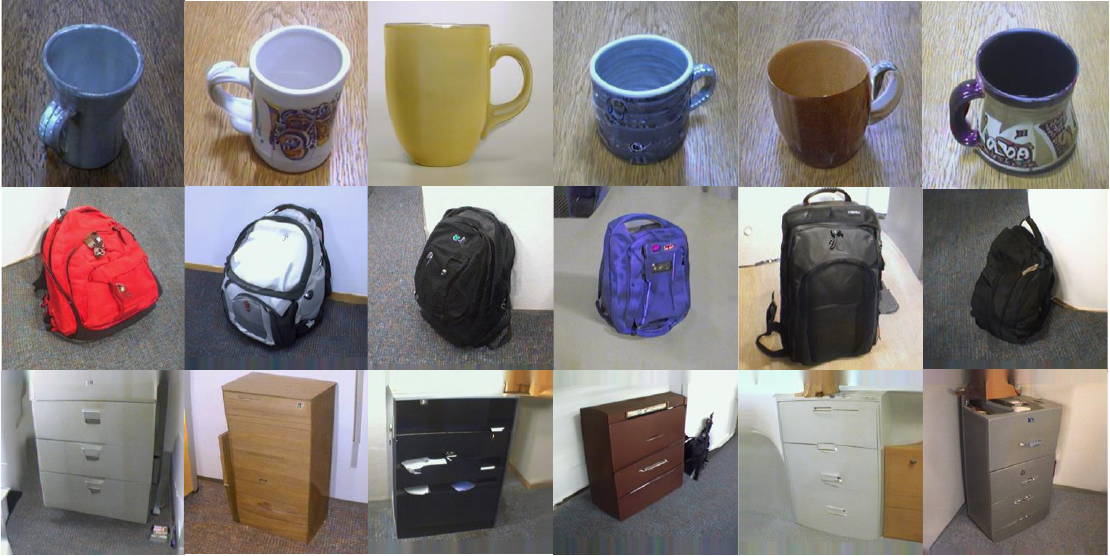}
    \label{fig:Office31_3}
}
\caption{Real and generated images for the transfer task A$\rightarrow$W on the \textit{Office-31} dataset.}
\label{generate}
\end{figure*}

\begin{figure}[!t]
\centering
\!\!
\subfigure[Real source samples]{
    \includegraphics[width=0.23\textwidth]{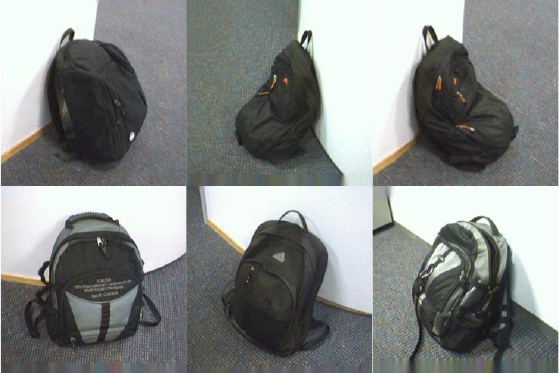}
    \label{fig:spurious1}
}
\subfigure[Generated target samples]{
    \includegraphics[width=0.23\textwidth]{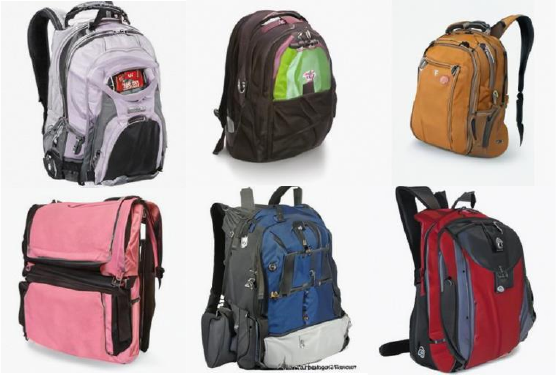}
    \label{fig:spurious2}
}\!\!
\vskip -.1in
\caption{Real source images and generated target images for the transfer task W$\rightarrow$A on the \textit{Office-31} dataset.}
\label{fig:spurious}
\end{figure}

\begin{figure}[!t]
\centering
\!\!
\subfigure[Real source samples.]{
    \includegraphics[width=0.087\textwidth]{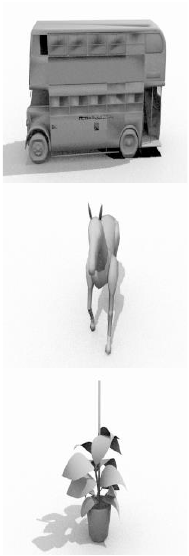}
    \label{fig:visda_1}
}
\subfigure[Real target samples.]{
    \includegraphics[width=0.087\textwidth]{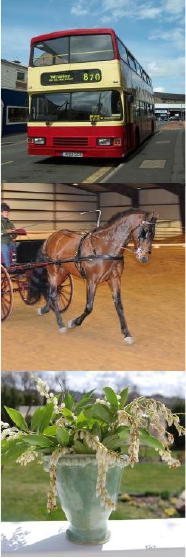}
    \label{fig:visda_2}
} 
\subfigure[Generated target samples.]{
    \includegraphics[width=0.262\textwidth]{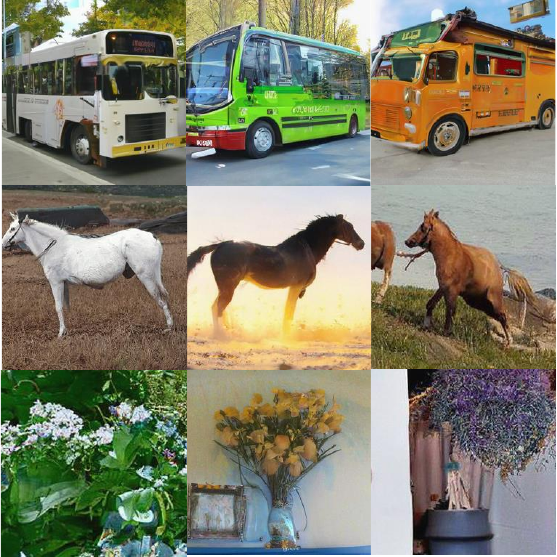}
    \label{fig:visda_3}
}\!\!
\vskip -.1in
\caption{Real and generated images for the transfer task Synthetic$\rightarrow$Real on the \textit{VisDA} dataset.}
\label{visda_generate}
\end{figure}

\begin{figure}[!ht]
\centering
\!\!
\subfigure[Real source samples.]{
    \includegraphics[width=0.087\textwidth]{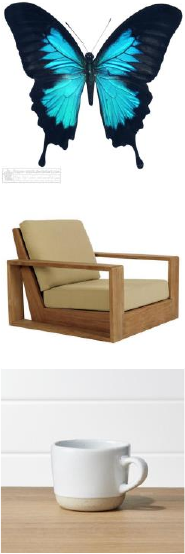}
    \label{fig:minidomainnet_1}
}
\subfigure[Real target samples.]{
    \includegraphics[width=0.087\textwidth]{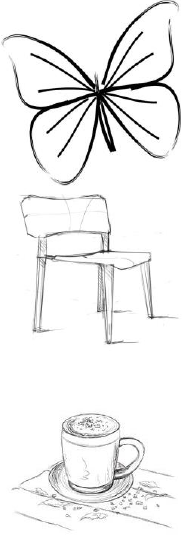}
    \label{fig:minidomainnet_2}
}
\subfigure[Generated target samples.]{
    \includegraphics[width=0.262\textwidth]{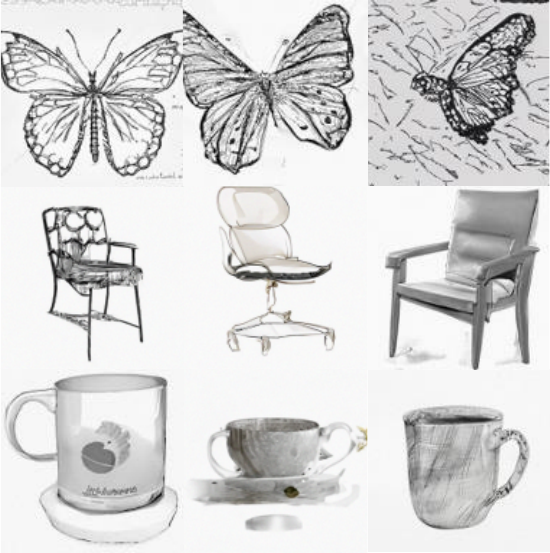}
    \label{fig:minidomainnet_3}
}\!\!
\vskip -.1in
\caption{Real and generated images for the transfer task R$\rightarrow$S on the \textit{miniDomainNet} dataset.}
\label{fig:minidomainnet}
\end{figure}

\begin{figure*}[!h]
\centering
\includegraphics[width=0.9\textwidth]{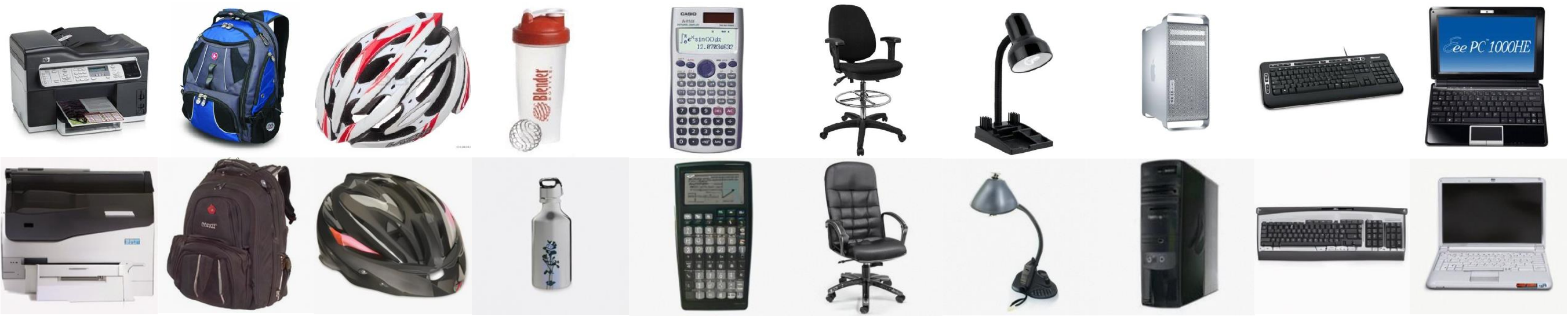}
\caption{Real source images (top) and generated source images (bottom) for the transfer task A$\rightarrow$W on the \textit{Office-31} dataset.}
\label{office31_img}
\end{figure*}

\begin{figure}[t]
\centering
\subfigure[A$\rightarrow$W]{
    \includegraphics[width=0.141\textwidth]{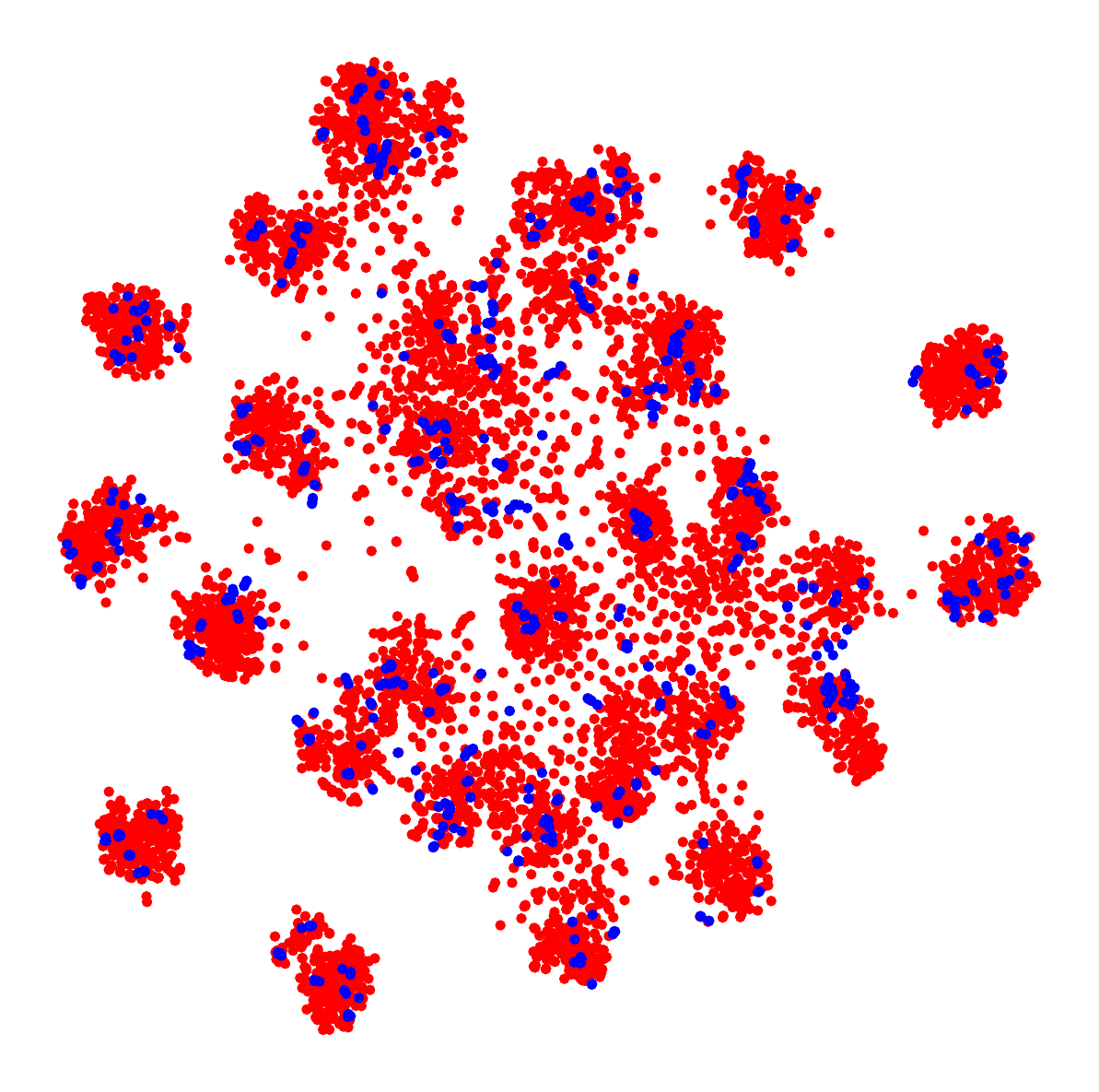}
    \label{fig:tsne_a2w}
}
\subfigure[D$\rightarrow$W]{
    \includegraphics[width=0.141\textwidth]{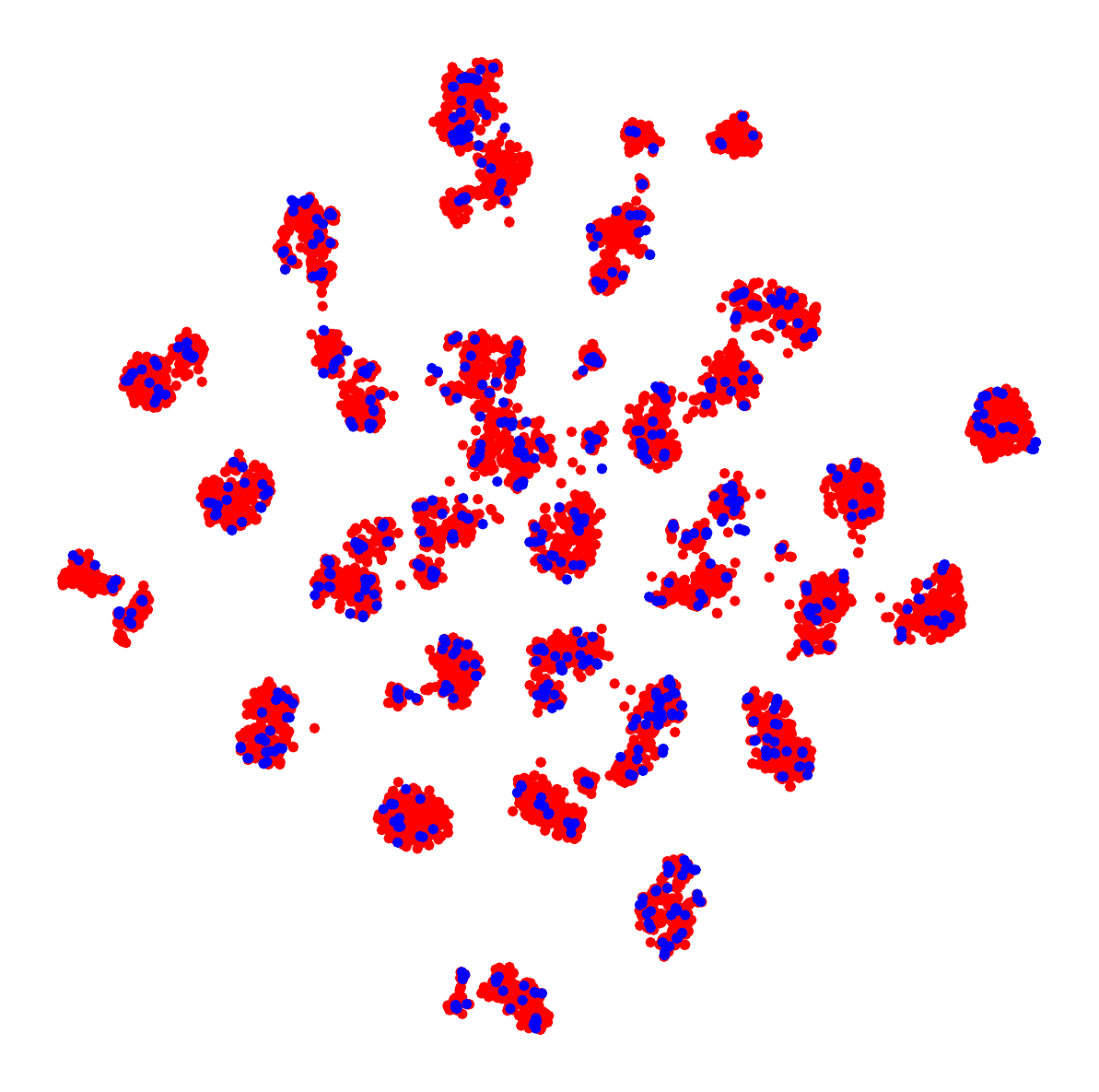}
    \label{fig:tsne_d2w}
} 
\subfigure[W$\rightarrow$D]{
    \includegraphics[width=0.141\textwidth]{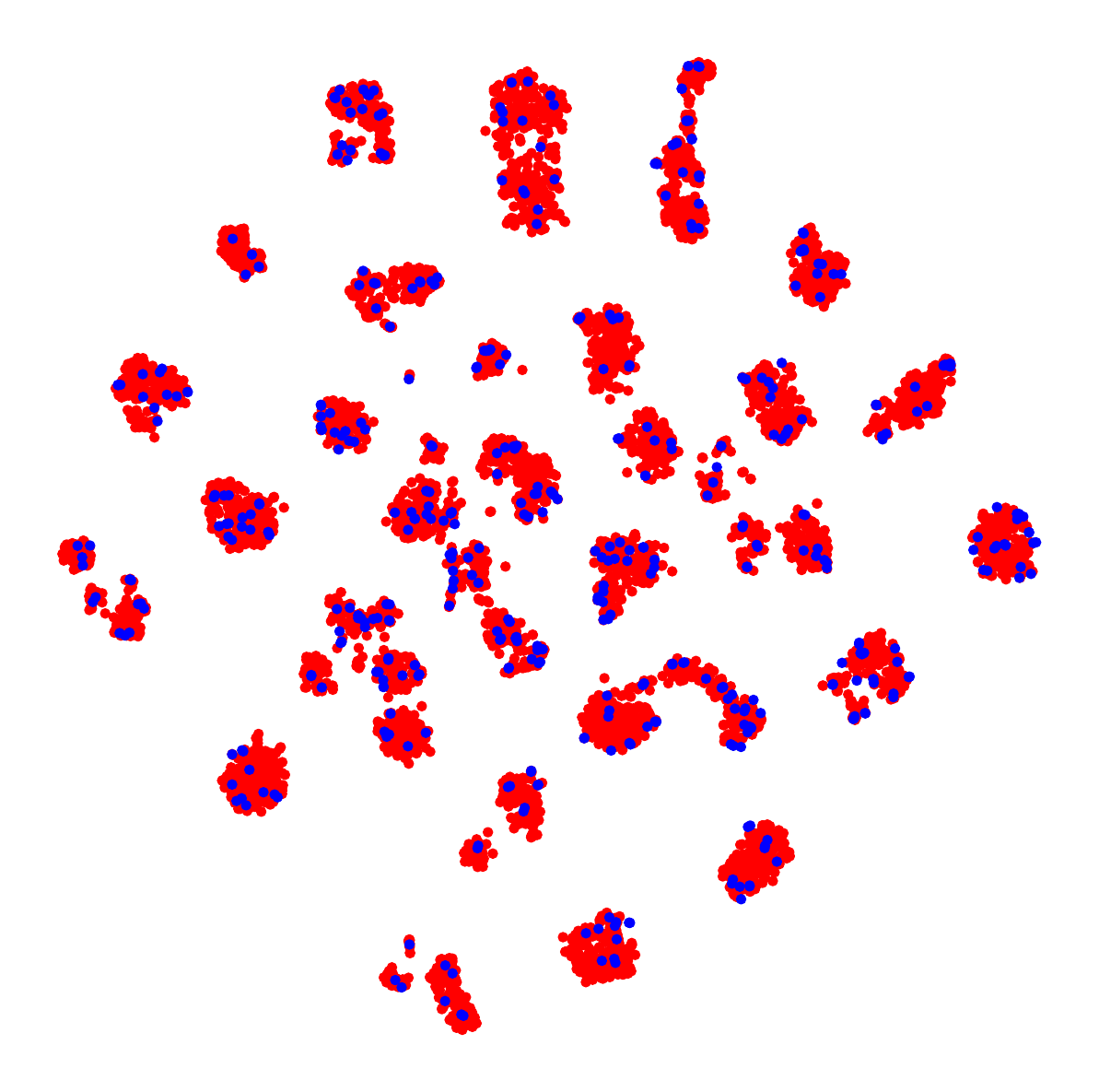}
    \label{fig:tsne_w2d}
} 
\subfigure[A$\rightarrow$D]{
\includegraphics[width=0.141\textwidth]{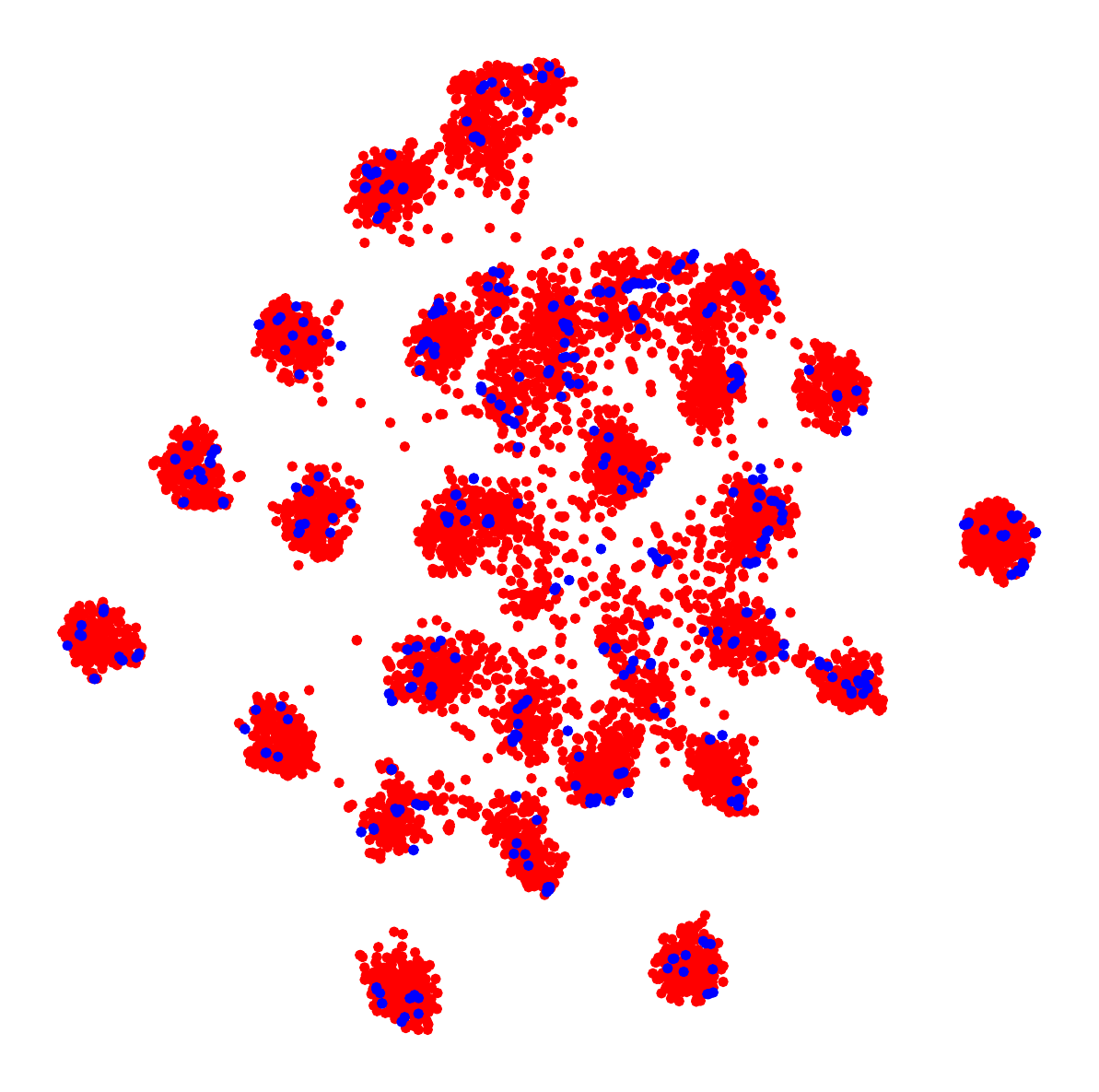}
    \label{fig:tsne_a2d}
} 
\subfigure[D$\rightarrow$A]{
\includegraphics[width=0.141\textwidth]{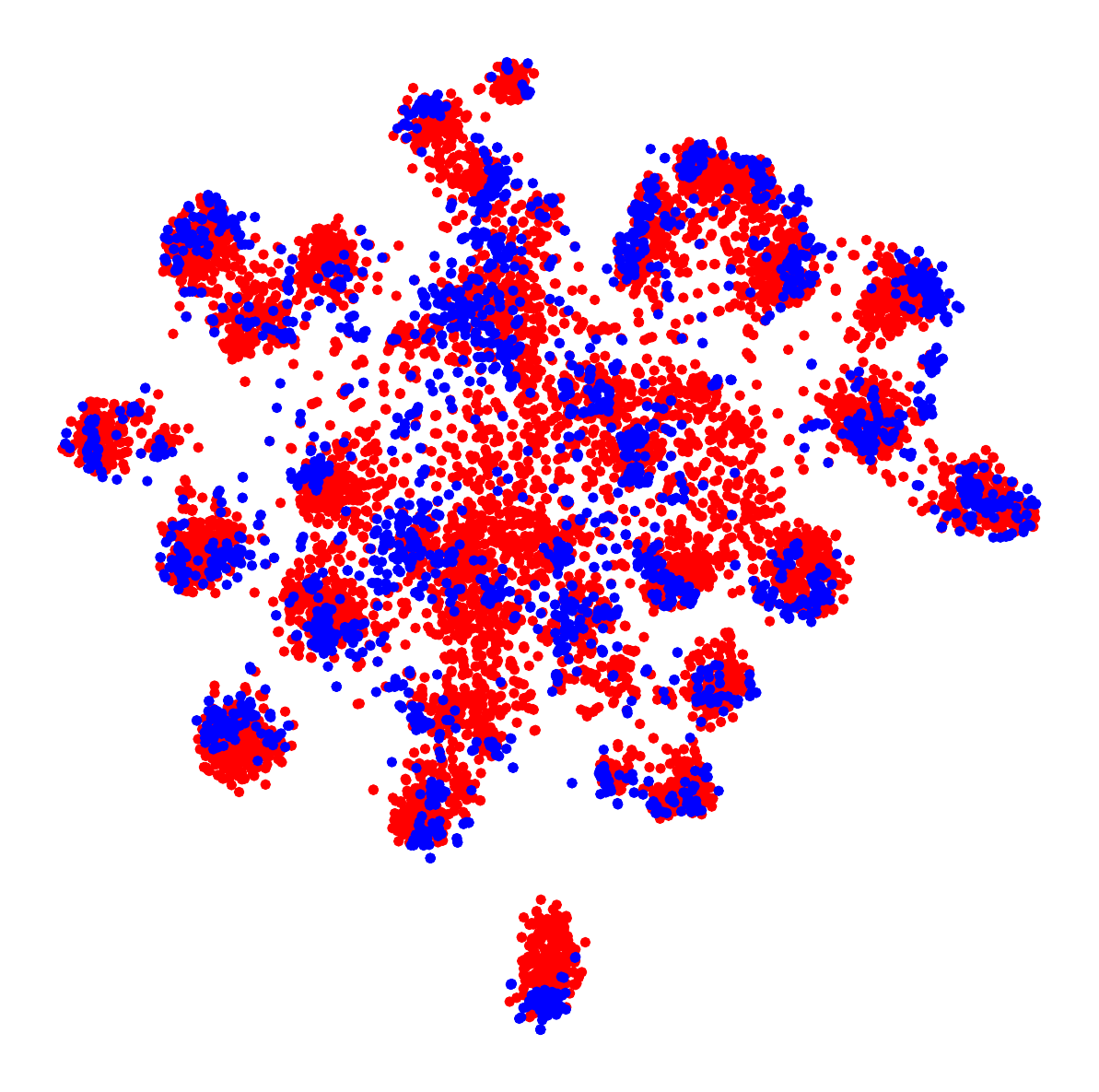}
    \label{fig:tsne_d2a}
} 
\subfigure[W$\rightarrow$A]{
    \includegraphics[width=0.141\textwidth]{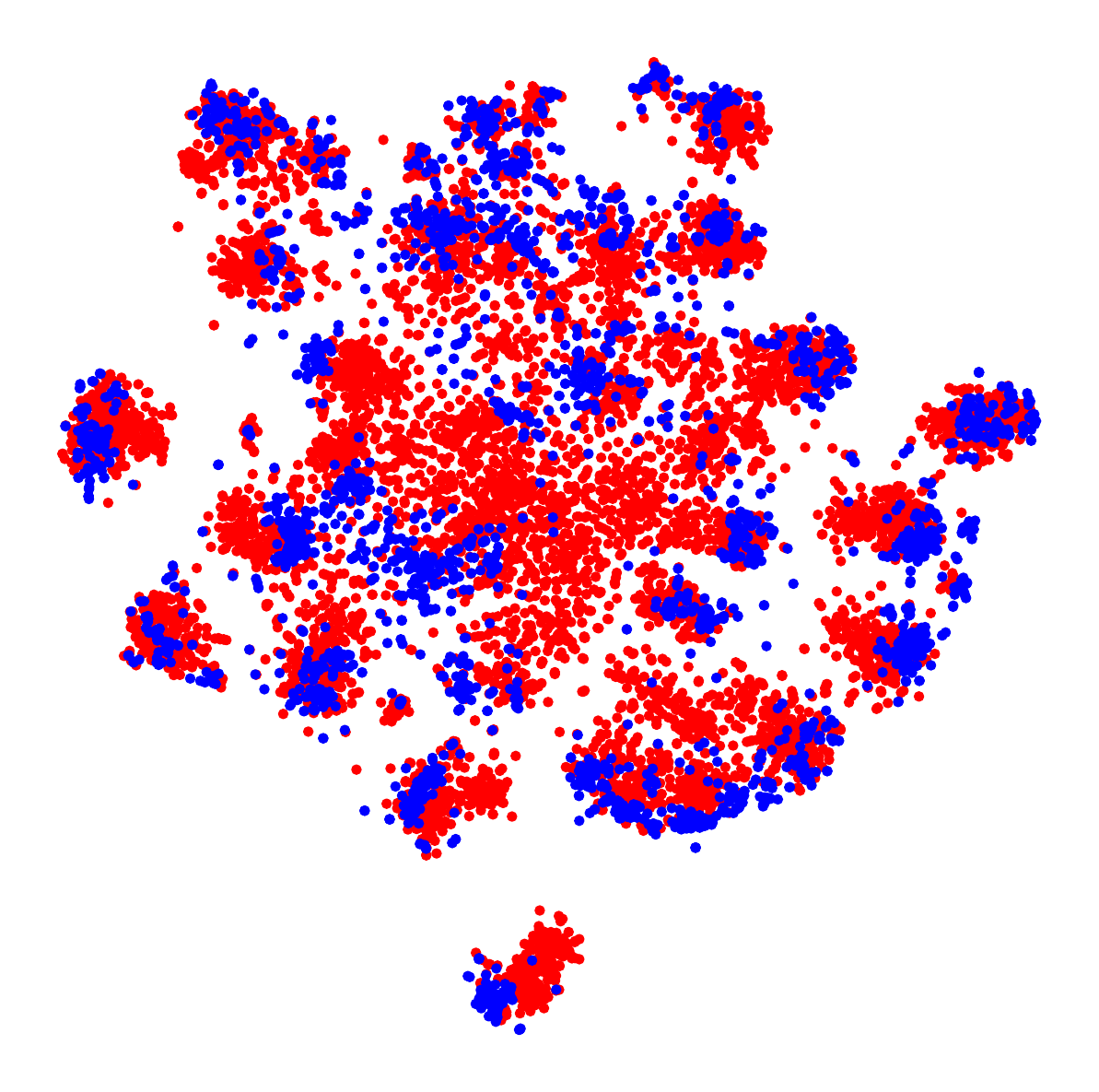}
    \label{fig:tsne_w2a}
} 
\caption{t-SNE visualization of six tasks on the \textit{Office-31} dataset. Red and blue points denote generated target samples and original target samples, respectively. 
Best viewed in color.}
\label{tSNE}
\end{figure}

\begin{table}[t]\small
\centering 
\caption{$\mathcal{A}$-distance across domains on six tasks of the \textit{Office-31} dataset.}
\vskip -.1in
\label{Adistance}
\resizebox{.48\textwidth}{!}{
\begin{tabular}{c @{\hskip 0.1in} ccccccc}
\toprule 
& A$\rightarrow$W & D$\rightarrow$W & W$\rightarrow$D & A$\rightarrow$D & D$\rightarrow$A & W$\rightarrow$A & Average \\
\midrule
$\mathcal{D}_s$ and $\mathcal{D}_t$ & 1.84 & 0.71 & 0.89 & 1.90 & 1.68 & 1.90 & 1.54\\
$\mathcal{D}_g$ and $\mathcal{D}_t$ & 0.45 & 0.33 & 0.51 & 0.57 & 1.20 & 1.31 & 0.73\\
$\mathcal{D}_{\hat{s}}$ and $\mathcal{D}_t$ & 0.64 & 0.51 & 0.52 & 0.65 & 1.31 & 1.39 & 0.84 \\
\bottomrule
\end{tabular}
}
\end{table}

\begin{figure}[t]
\centering
\subfigure[A$\rightarrow$W]{
    \includegraphics[width=0.15\textwidth]{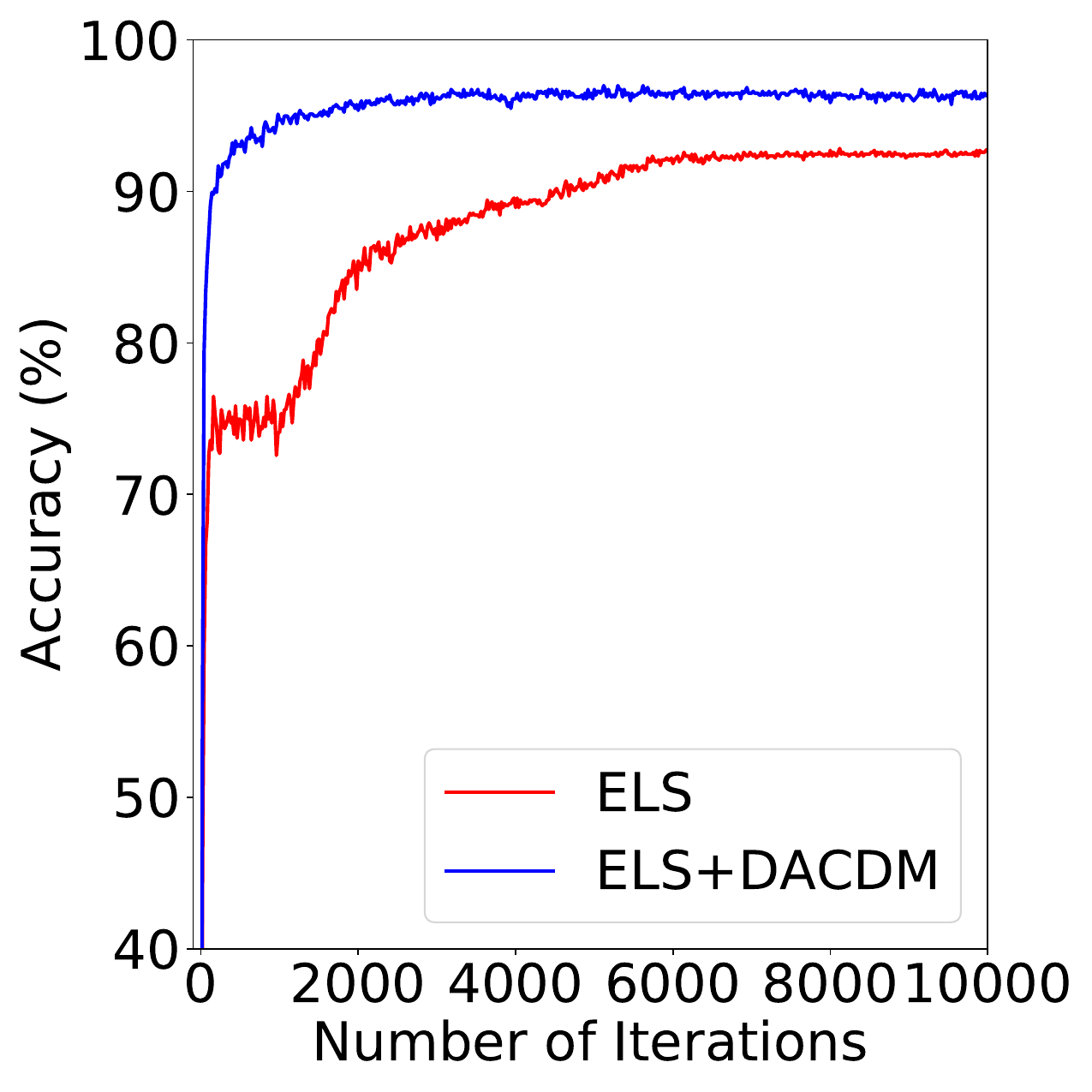}
    \label{fig:acc_a2w}
}\hspace{-3mm}
\subfigure[D$\rightarrow$W]{
    \includegraphics[width=0.15\textwidth]{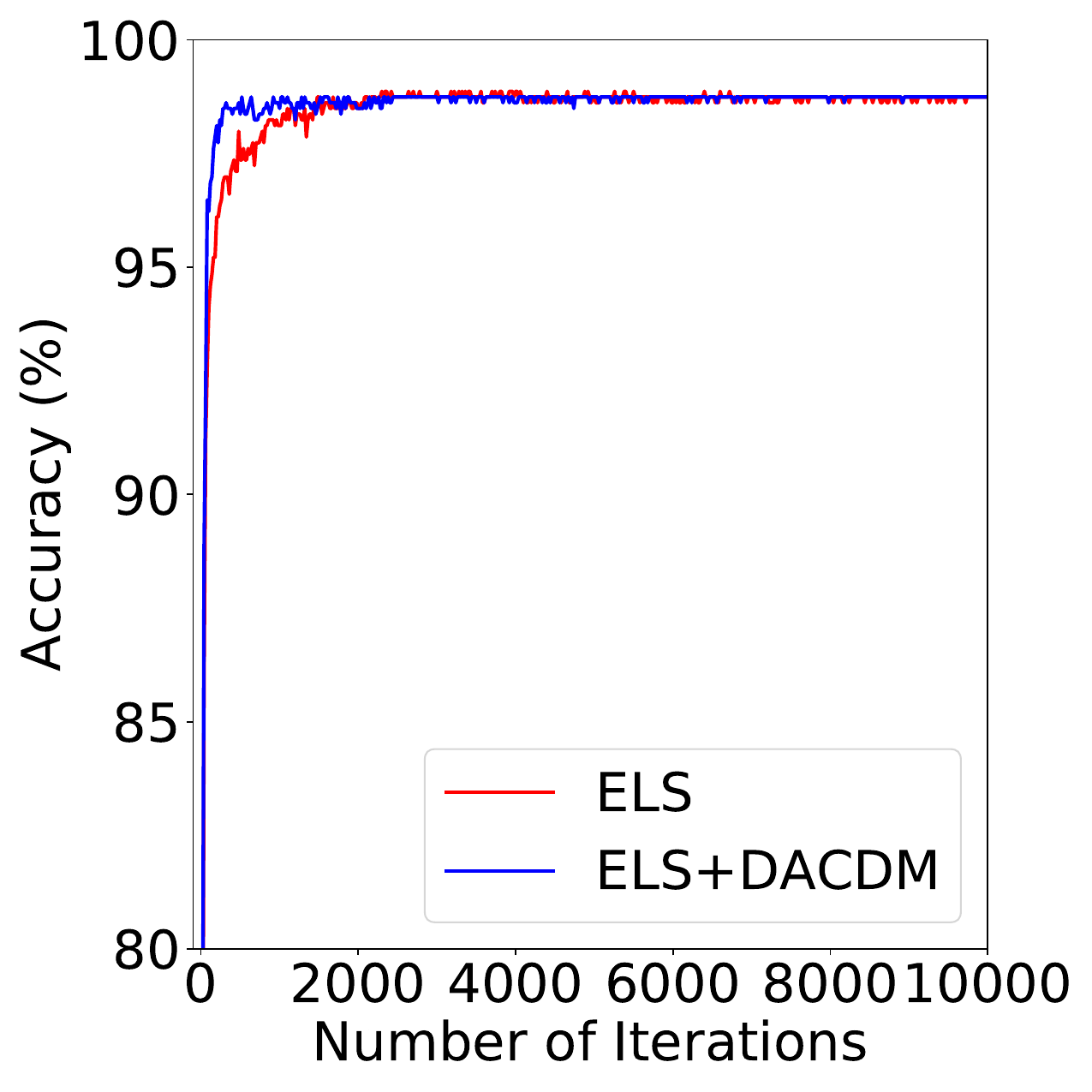}
    \label{fig:acc_d2w}
}\hspace{-3mm}
\subfigure[W$\rightarrow$D]{
    \includegraphics[width=0.15\textwidth]{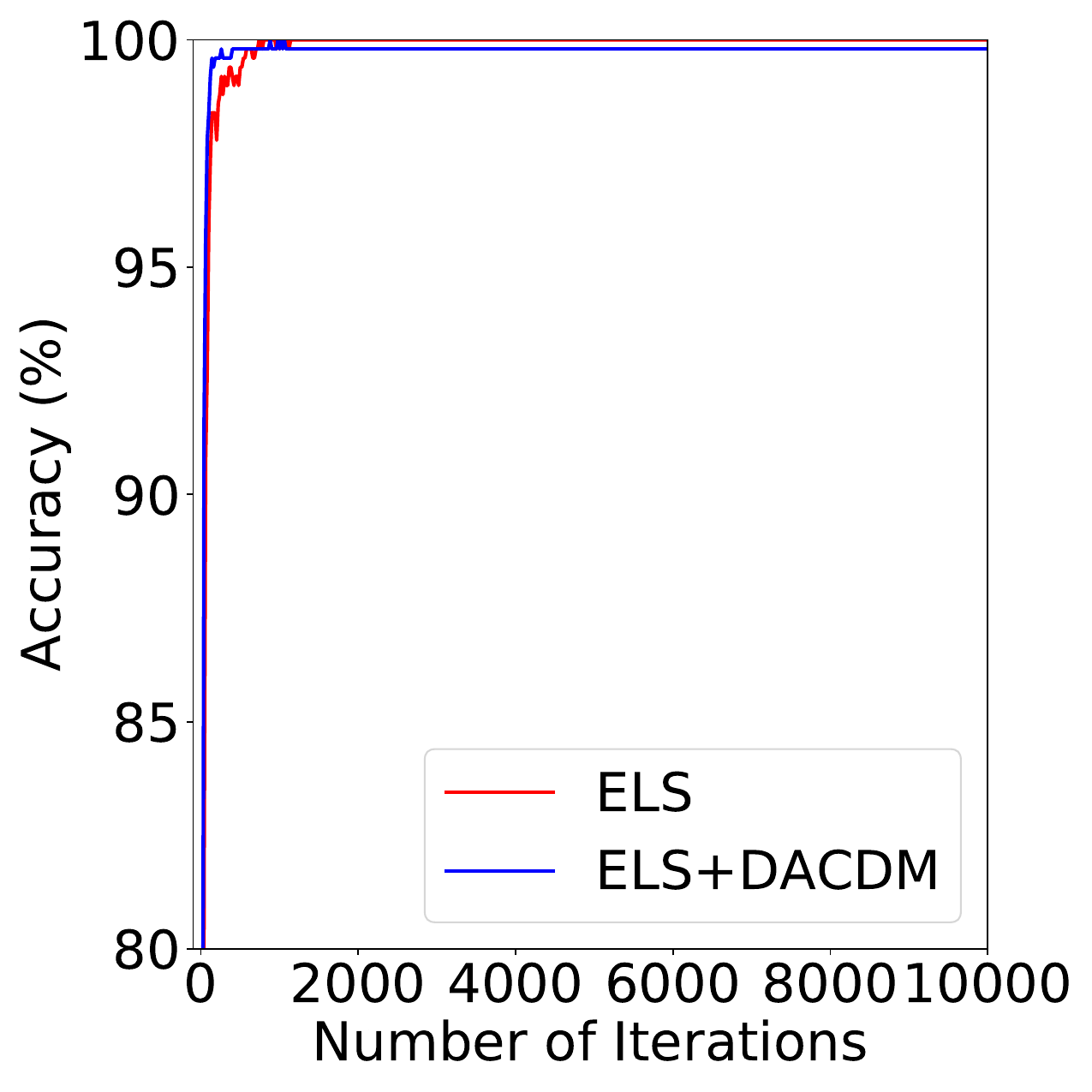}
    \label{fig:acc_w2d}
} 
\subfigure[A$\rightarrow$D]{
    \includegraphics[width=0.15\textwidth]{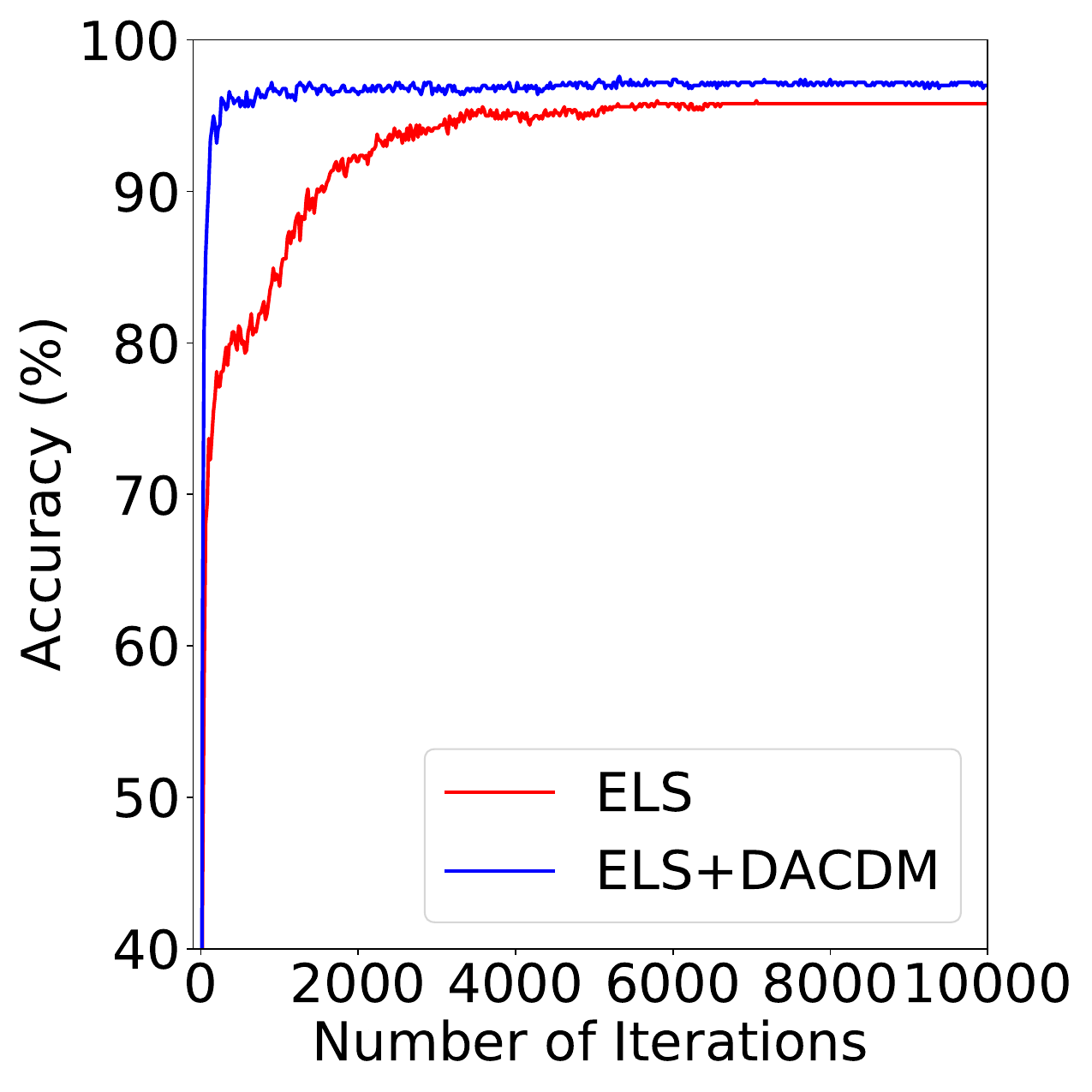}
    \label{fig:acc_a2d}
}\hspace{-3mm} 
\subfigure[D$\rightarrow$A]{
    \includegraphics[width=0.15\textwidth]{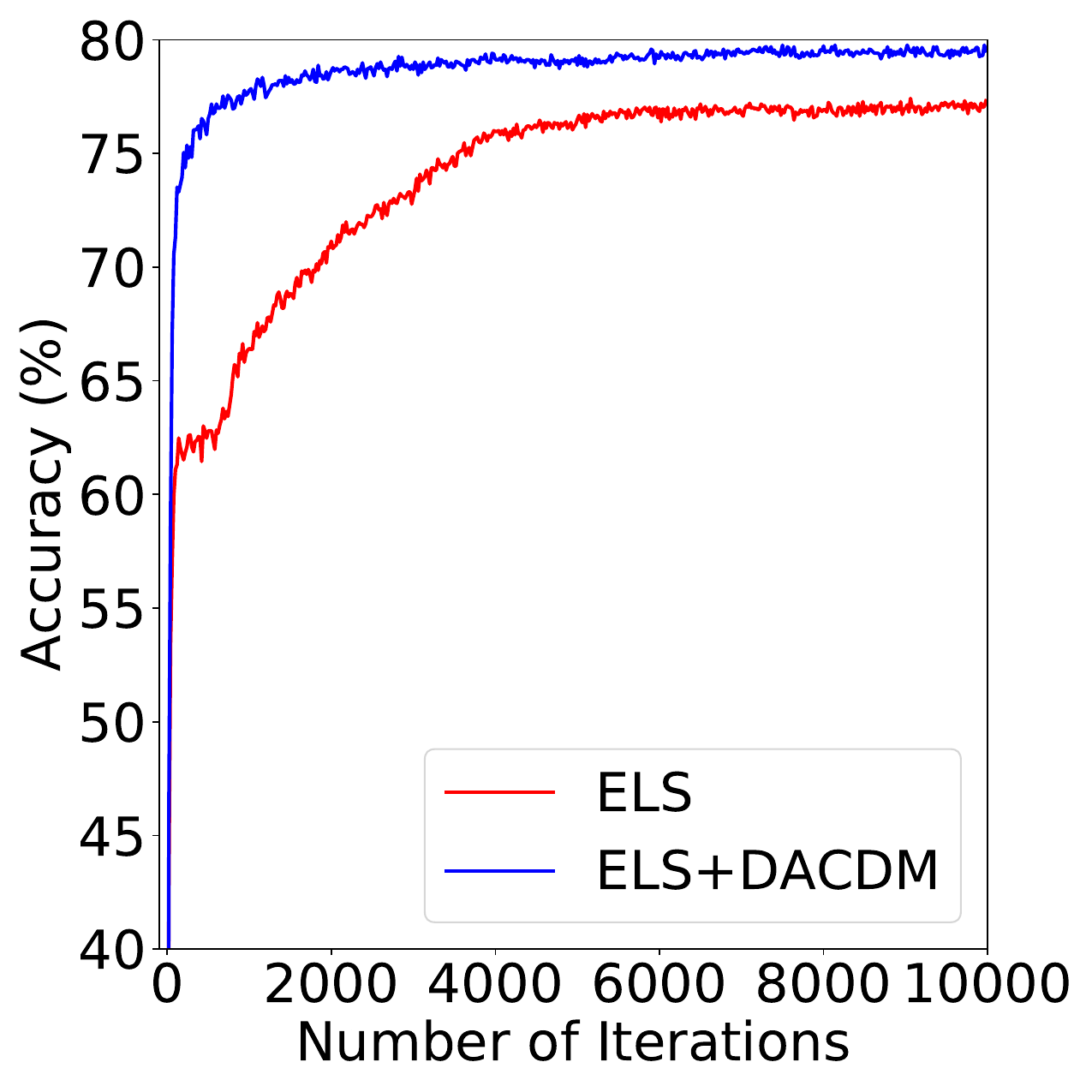}
    \label{fig:acc_d2a}
}\hspace{-3mm}
\subfigure[W$\rightarrow$A]{
    \includegraphics[width=0.15\textwidth]{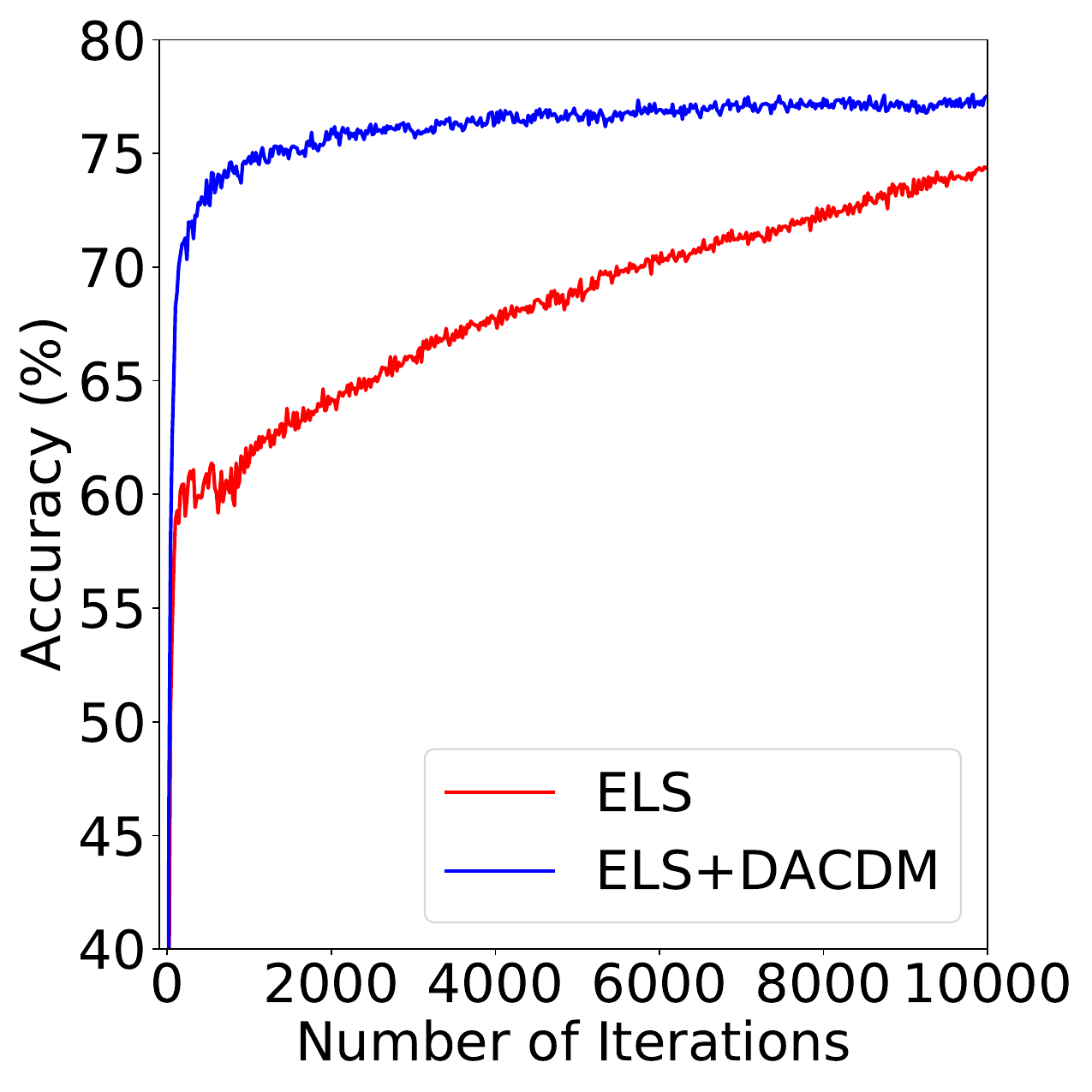}
    \label{fig:acc_w2a}
} 
\caption{Test accuracy (\%) of ELS and ELS+DACDM w.r.t. the number of iterations on the \textit{Office-31} dataset.}
\label{ELS_acc}
\end{figure}

\begin{figure}[t]
\centering
\subfigure[A$\rightarrow$W]{
    \includegraphics[width=0.15\textwidth]{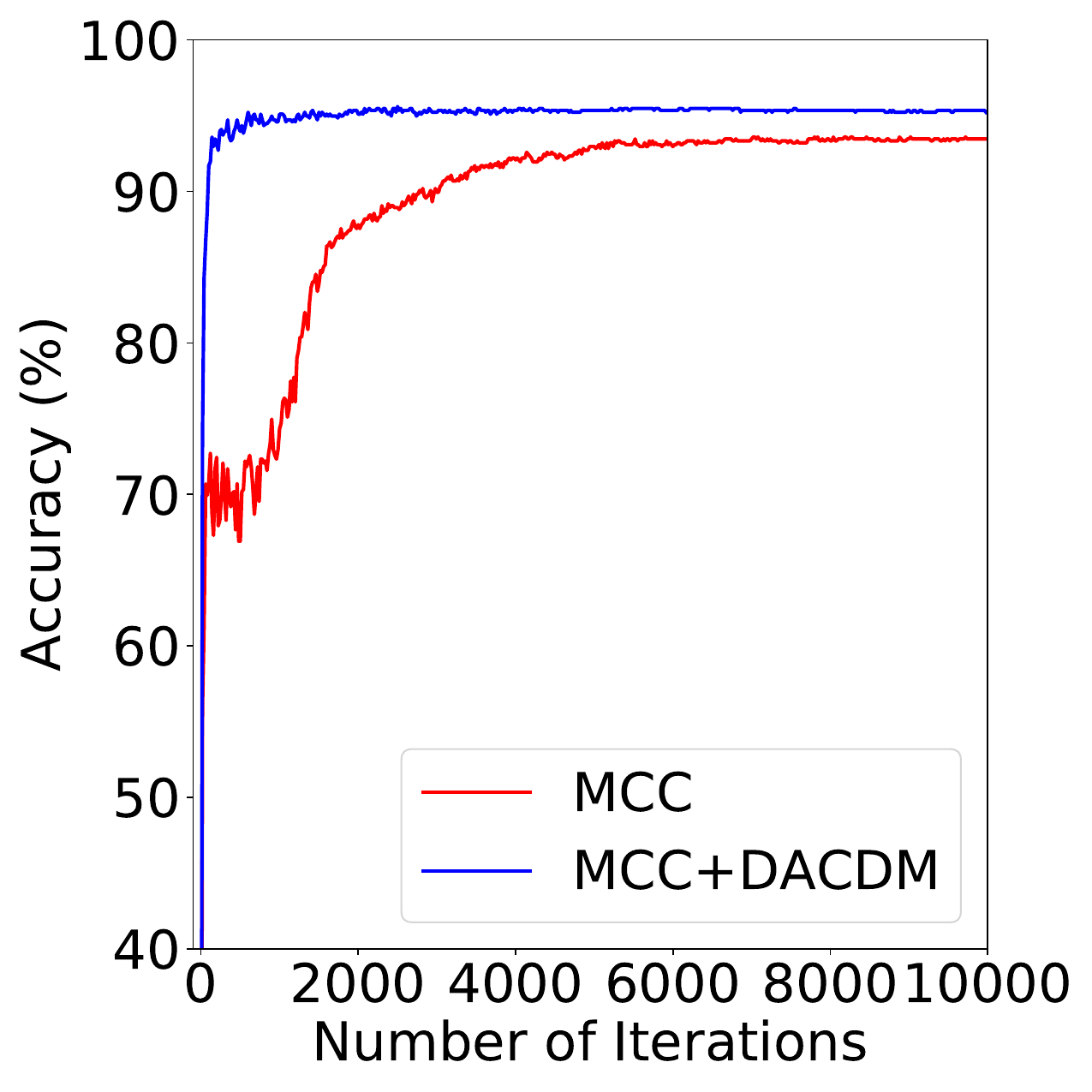}
    \label{fig:acc_mcc_a2w}
}\hspace{-3mm}
\subfigure[D$\rightarrow$W]{
    \includegraphics[width=0.15\textwidth]{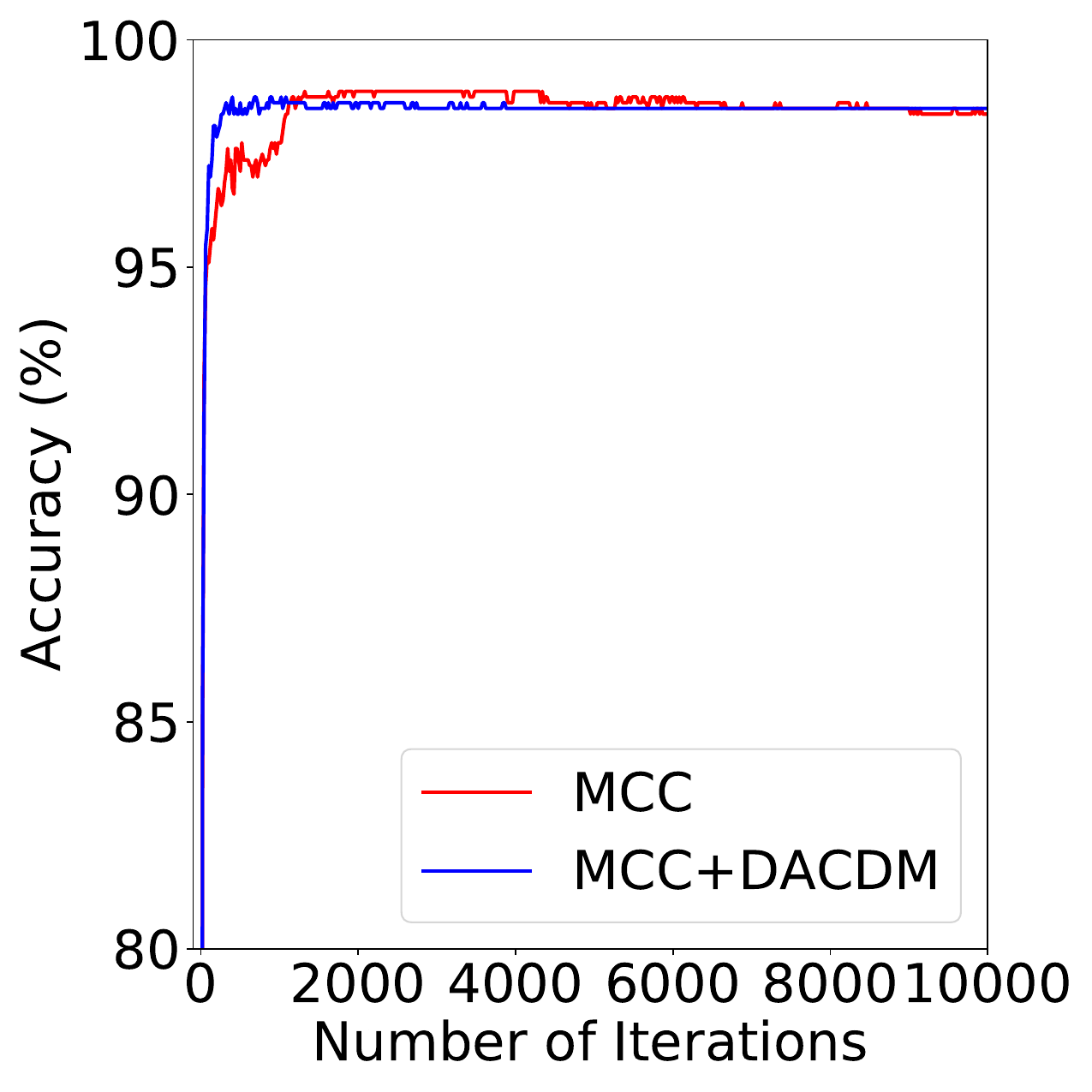}
    \label{fig:acc_mcc_d2w}
} \hspace{-3mm}
\subfigure[W$\rightarrow$D]{
    \includegraphics[width=0.15\textwidth]{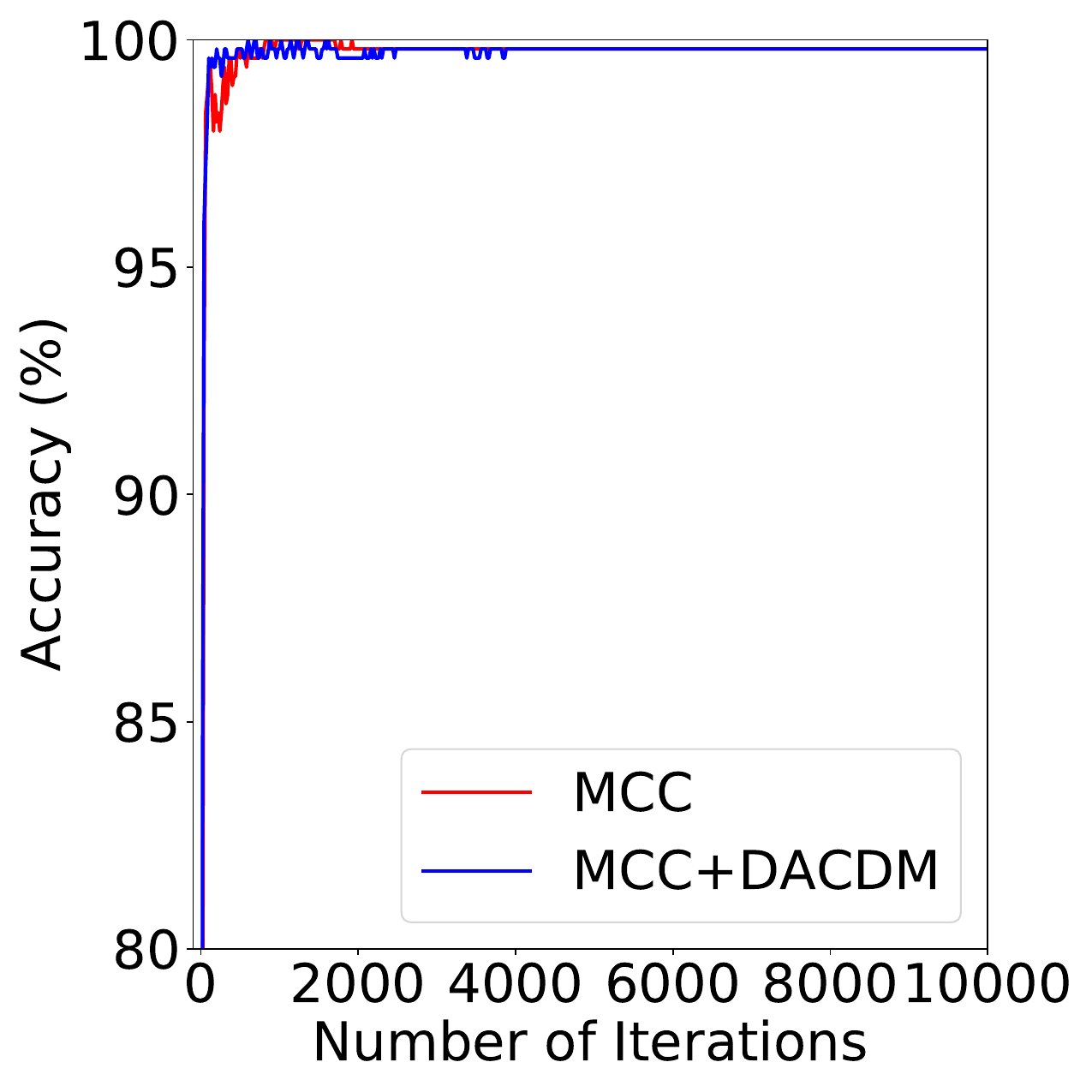}
    \label{fig:acc_mcc_w2d}
} 
\subfigure[A$\rightarrow$D]{
    \includegraphics[width=0.15\textwidth]{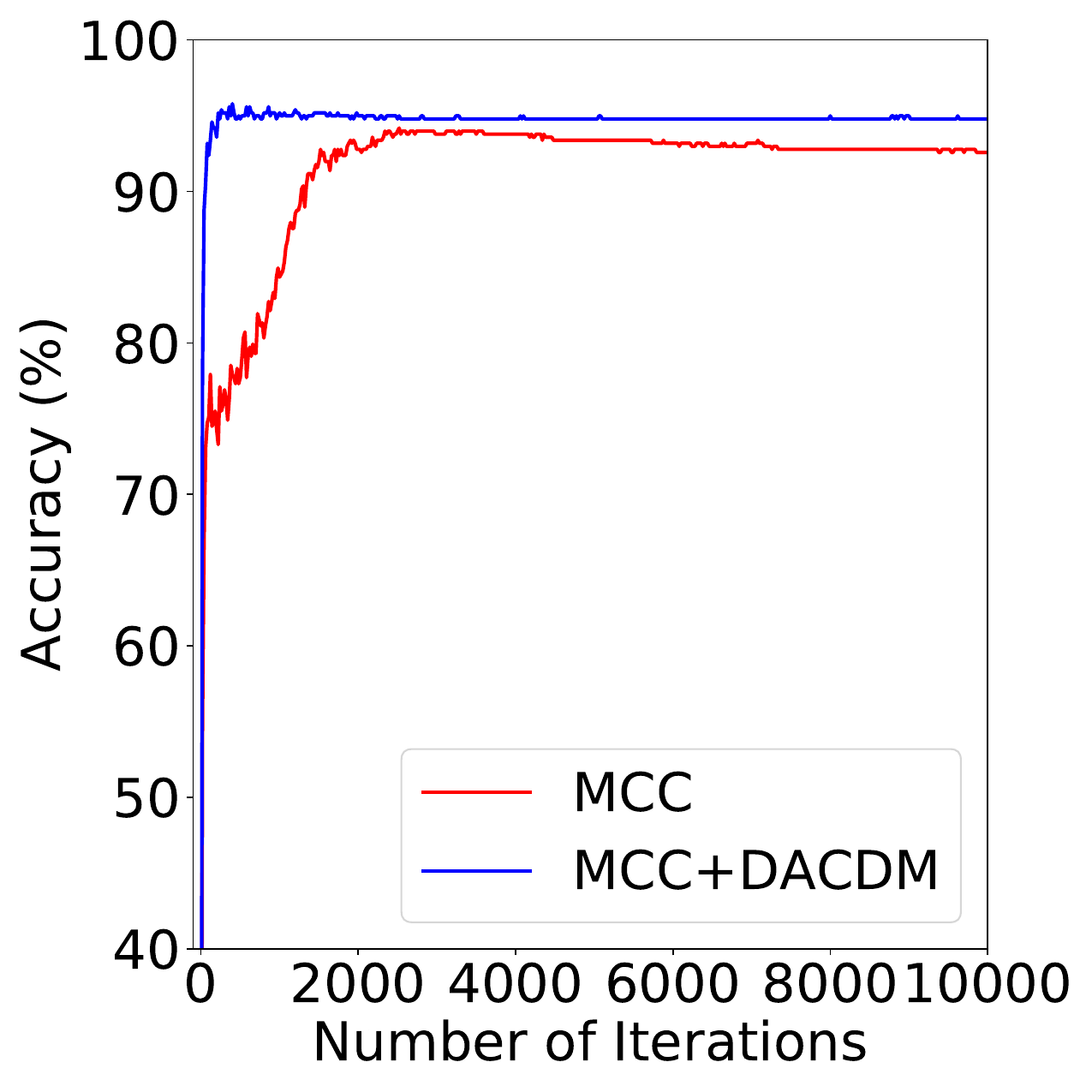}
    \label{fig:acc_mcc_a2d}
} \hspace{-3mm}
\subfigure[D$\rightarrow$A]{
    \includegraphics[width=0.15\textwidth]{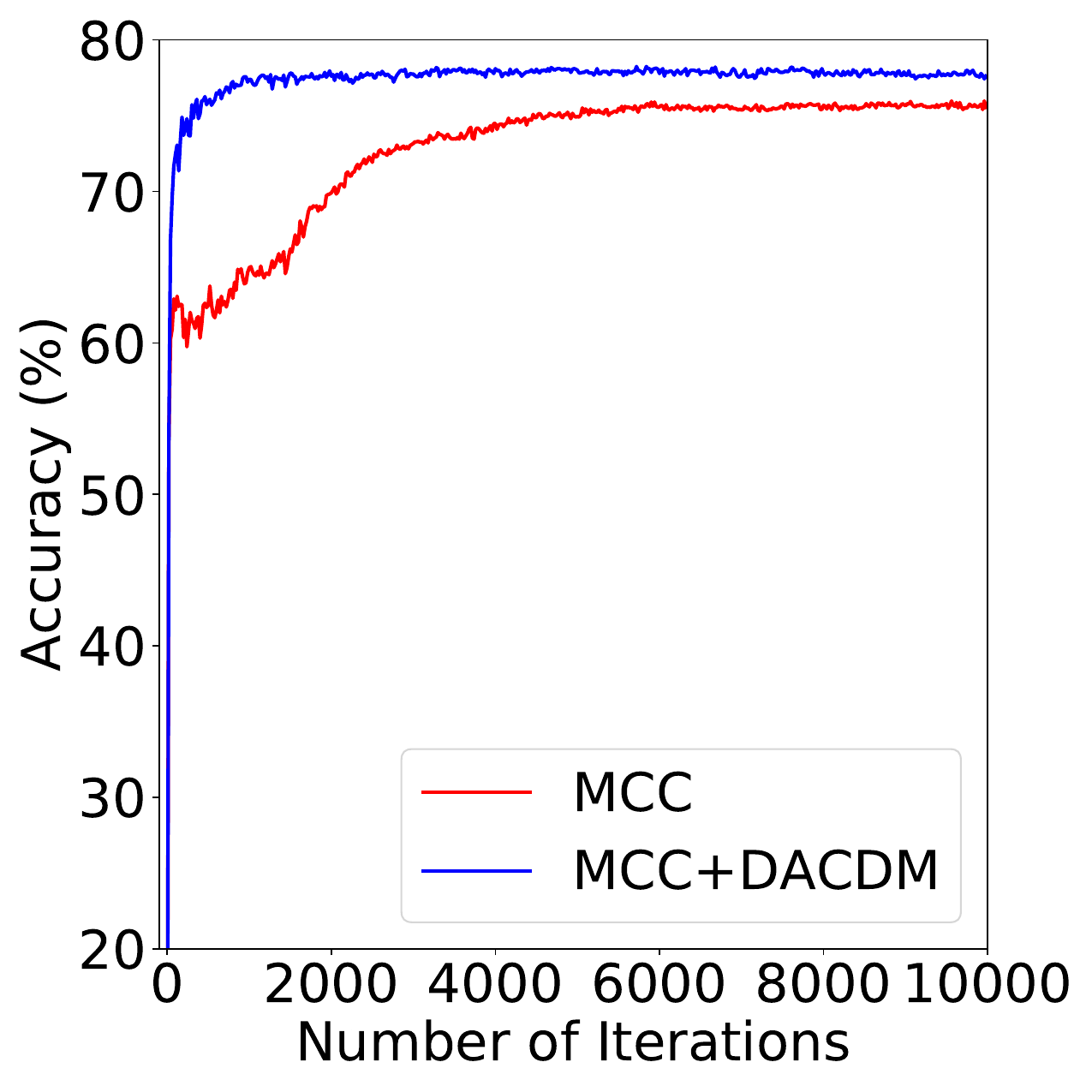}
    \label{fig:acc_mcc_d2a}
} \hspace{-3mm}
\subfigure[W$\rightarrow$A]{
    \includegraphics[width=0.15\textwidth]{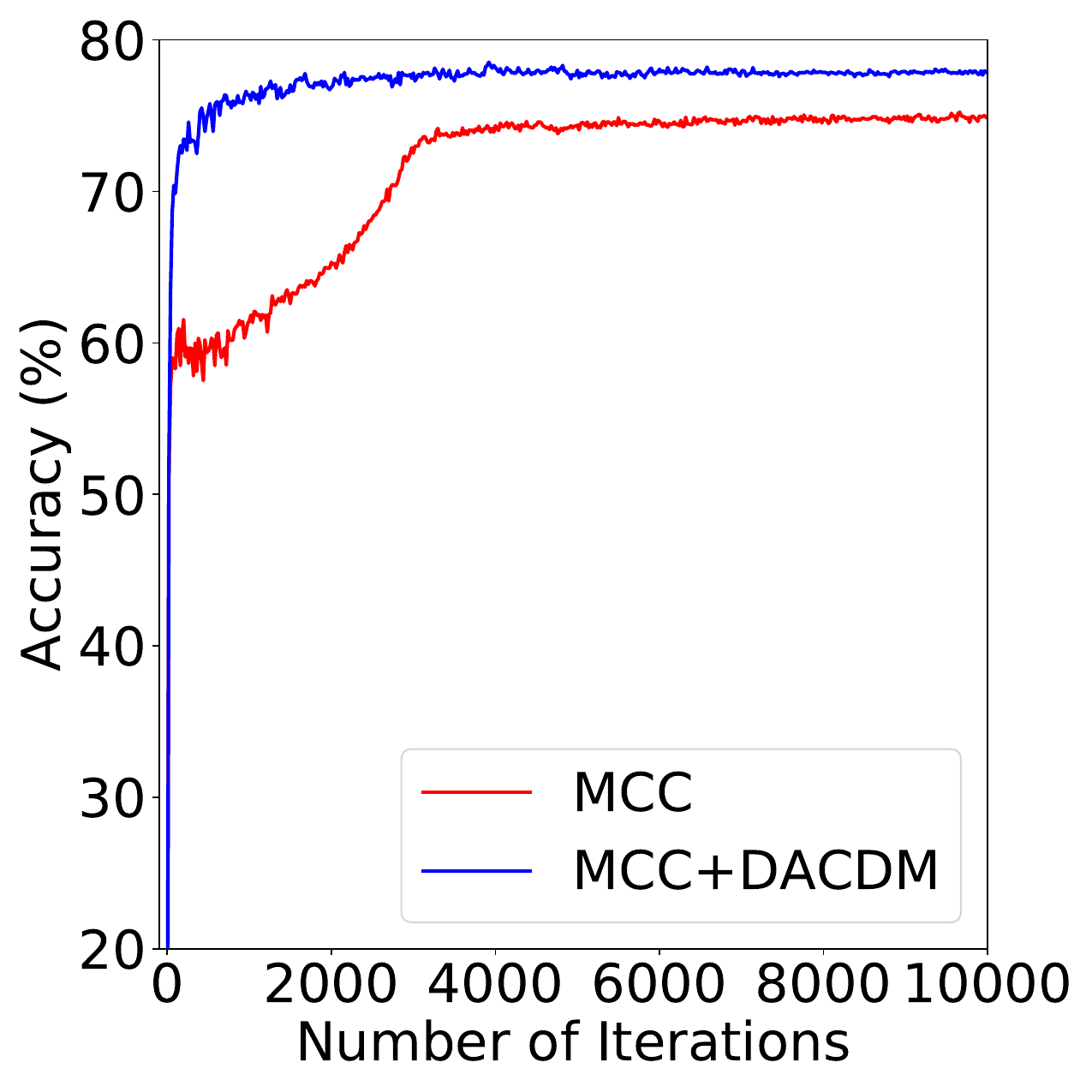}
    \label{fig:acc_mcc_w2a}
} 
\caption{Test accuracy (\%) of MCC and MCC+DACDM w.r.t. the number of iterations on the \textit{Office-31} dataset.}
\label{MCC_acc}
\end{figure}

In this section, we show fidelity and diversity of the generated samples.
Fig.~\ref{generate1} shows the generated images of the transfer task Ar$\rightarrow$Cl on \textit{Office-Home}.
As can be seen, the style of the generated samples 
(Fig.  \ref{fig:Office_home_3}) is similar to that of the real target samples (Fig.
	 \ref{fig:Office_home_2}), showing the effectiveness of DACDM. 
Moreover, the generated samples exhibit high fidelity and diversity,
and thus well simulate the distribution of the target domain. 
    Furthermore, the classes of generated samples are accurately controlled.

Fig.~\ref{generate} shows the generated target samples of the transfer task A$\rightarrow$W on \textit{Office-31}.
Again, the generated samples (Fig. \ref{fig:Office31_3}) are of high quality and have a
similar style as the real target samples (Fig. \ref{fig:Office31_2}).
    Furthermore, the generated images improve diversity of 
    the target domain images.
    For example, in the second row of Fig.~\ref{generate}, the real target domain has only black backpacks, while the generated images contain backpacks in various colors.
    One possible reason is that DACDM incorporates source samples during the
	 diffusion training process, allowing it to generate backpacks with more
	 colors related to the source domain while preserving the style of the target domain.
    Similarly, in the third row of Fig.~\ref{generate}, 
    the target samples contain only metal-style file cabinets,
    while the generated images have cabinets of more styles related to the source domain with backgrounds similar to the target domain.
    This demonstrates 
    domain guidance can generate diverse images for the target domain.

Traditional DA may suffer from spurious correlation in the source domain \cite{bao2022learning}.
For example, on the transfer task W$\to$A of \textit{Office-31}, all backpacks in
domain W are black (Fig.~\ref{fig:spurious1}), while backpacks in target domain A have various colors. 
When transferring from domain W to domain A, the UDA model may focus on recognizing black backpacks. 
Fig.~\ref{fig:spurious2} shows the generated target images. 
As can be seen, the generated backpacks have various colors.
By combining the generated data and source domain data, the accuracy on the class
`backpacks' is improved from $96.74\%$ to $97.83\%$ using ELS+DACDM, demonstrating that DACDM can help existing UDA methods to learn more transferrable features.

Fig. \ref{visda_generate} and \ref{fig:minidomainnet} show the generated images for the transfer tasks Synthetic$\to$Real and R$\to$S on \textit{VisDA-2017} and \textit{miniDomainNet}, respectively.
Again, the generated images have a similar style as the target samples,
demonstrating usefulness of the proposed DACDM.

With domain guidance, the trained DACDM can also generate source domain images.
Fig. \ref{office31_img} shows the real source domain images and generated source domain images for the transfer task A$\to$W on \textit{Office-31}.
As can be seen, their styles are similar, demonstrating effectiveness of domain guidance.

Fig.~\ref{tSNE} shows the t-SNE visualization \cite{van2008visualizing} of feature
embeddings for samples from the six transfer tasks on \textit{Office-31} based
on a pretrained \textit{ResNet-50} on \textit{ImageNet}.
As can be seen, embeddings of the generated target samples (in red) are close to
those of the real target samples (in blue), verifying usefulness of DACDM. 
Moreover, we use the $\mathcal{A}$-distance \cite{ben2006analysis} to measure the distributional discrepancy on \textit{Office-31}.
Table \ref{Adistance} shows the $\mathcal{A}$-distances of $\hD_s$, $\hD_g$ and $\hD_{\hat{s}}$ to $\hD_t$.
As can be seen, 
compared with $\hD_s$, 
$\hD_g$ and $\hD_{\hat{s}}$ are closer to $\hD_t$ on all six tasks, implying that the generated samples by DACDM make it easier to transfer from the augmented source domain $\hD_{\hat{s}}$
    to the target domain $\hD_t$ than the source domain $\hD_{s}$.

\begin{table*}[!t]\small
\centering
\caption{Ablation studies on the \textit{Office-31} dataset using \textit{ResNet-50} in terms of accuracy (\%). The best is in \textbf{bold}.}
\vskip -.1in
\label{ablation}
\setlength{\tabcolsep}{1.5mm}{
\begin{tabular}{ccccccc @{\hskip 0.2in} c}
\toprule 
& A$\rightarrow$W        & D$\rightarrow$W        & W$\rightarrow$D        & A$\rightarrow$D        & D$\rightarrow$A        & W$\rightarrow$A        & Average \\
\midrule
ELS \cite{zhang2023free} & 93.84 & 98.78 & \textbf{100.00} & 95.78 & 77.72 & 75.13 & 90.21 \\
ELS+PseudoLabel & 94.55 & 98.74 & \textbf{100.00}  & 95.71 & 78.95 & 75.45 & 90.57 \\
ELS+DACDM (w/o domain guidance) & 94.48 & \textbf{99.00} & \textbf{100.00}  & 96.92 & 79.09 & 76.88 & 91.06  \\
ELS+DACDM ($\hD_g\to\hD_t$) & 96.56 & 98.74 & \textbf{100.00}  & 97.06 & 79.66 & 76.96 & 91.50 \\			
ELS+DACDM & \textbf{96.90} & 98.91 & \textbf{100.00} & \textbf{97.46} & \textbf{79.79} & \textbf{77.74} & \textbf{91.80} \\
\bottomrule
\end{tabular}}
\end{table*}

\begin{figure*}[!tbh]
\centering
\subfigure[\textit{Office-31}.]{
    \includegraphics[width=0.22\textwidth]{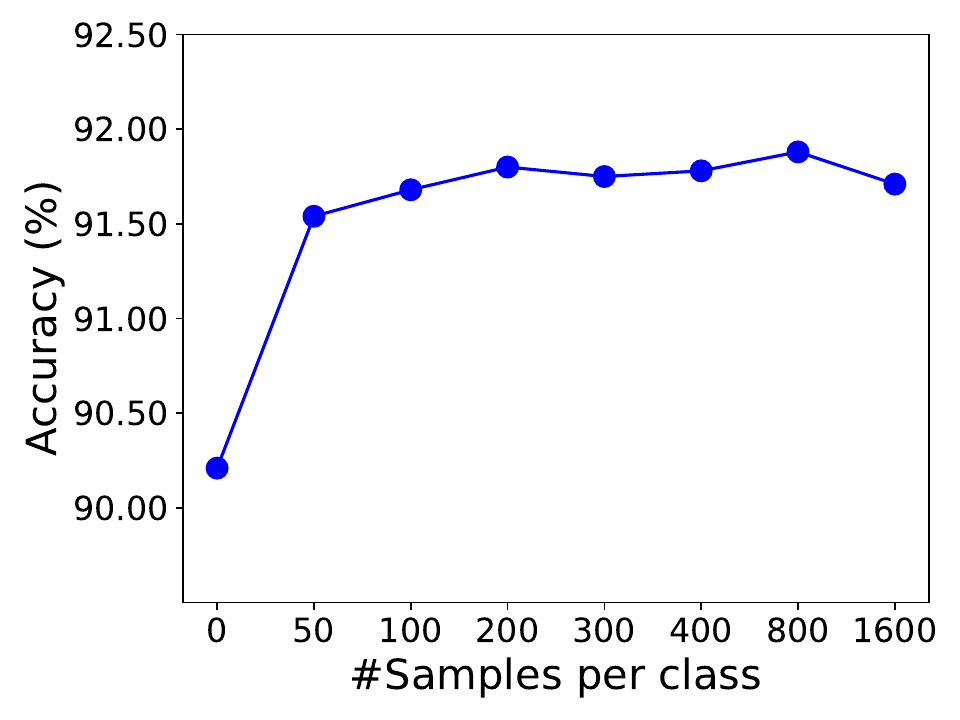}
    \label{fig:Office_31_number}
}
\subfigure[\textit{Office-Home}.]{
    \includegraphics[width=0.22\textwidth]{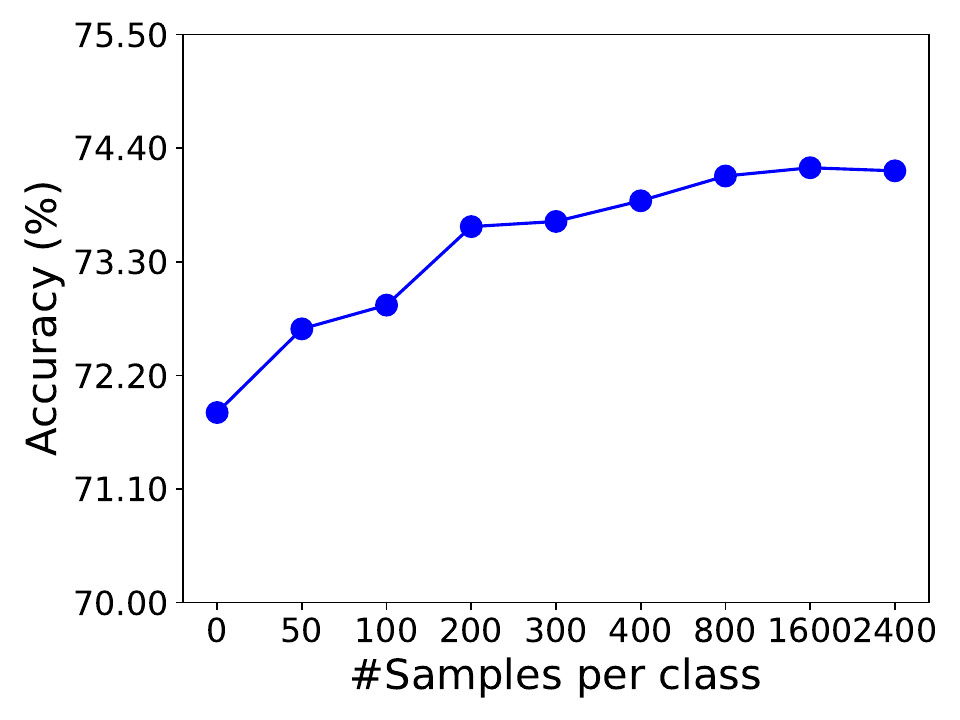}
    \label{fig:Office_Home_number}
} 
\subfigure[\textit{VisDA-2017}.]{
    \includegraphics[width=0.22\textwidth]{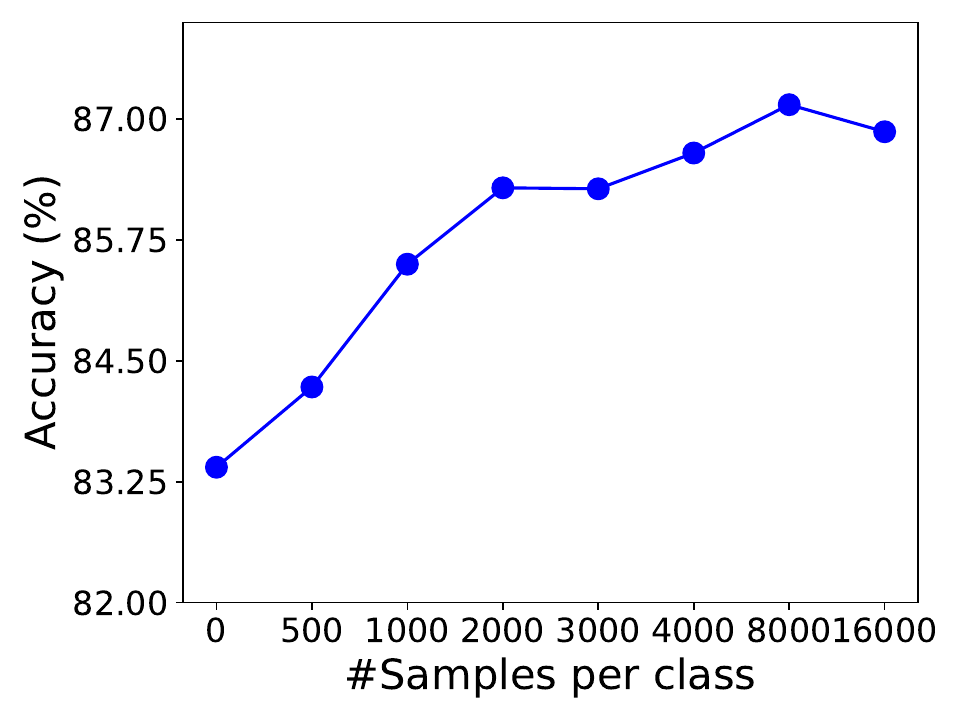}
    \label{fig:VisDA_number}
} 
\subfigure[\textit{miniDomainNet.}]{
    \includegraphics[width=0.22\textwidth]{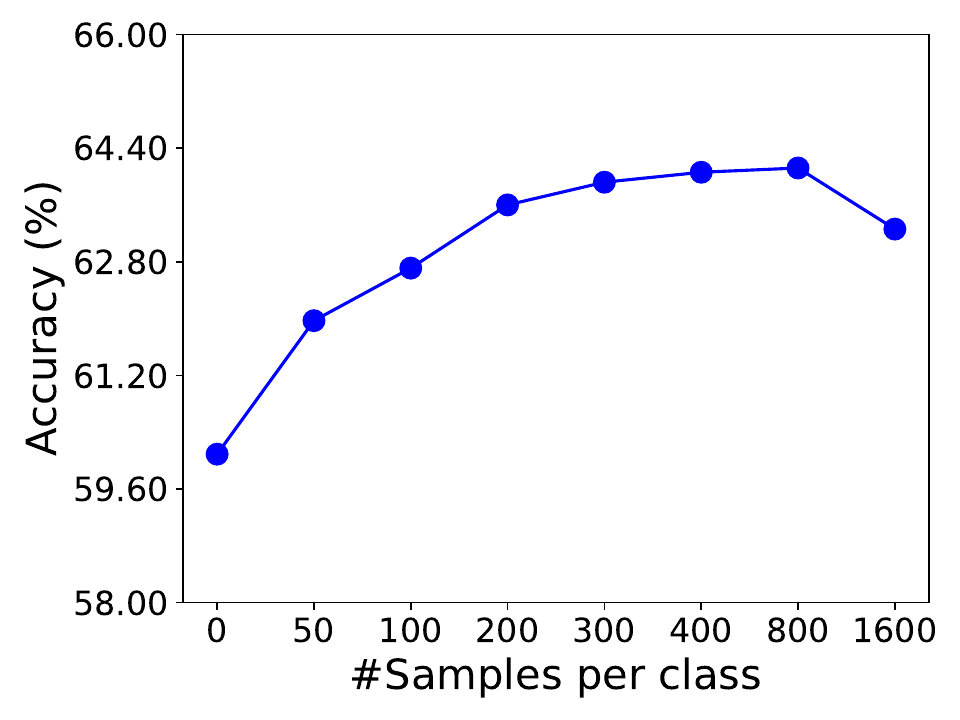}
    \label{fig:miniDomainNet_number}
}
\vskip -.1in
\caption{Accuracy w.r.t. samples per class on the \textit{Office-31}, \textit{Office-Home}, \textit{VisDA-2017}, and \textit{miniDomainNet} datasets.}
\label{Sensitivity}
\end{figure*}

\begin{figure}[!t]
\centering
\!\!
\subfigure[ELS+DACDM (w/o domain guidance).]{
    \includegraphics[width=0.23\textwidth]{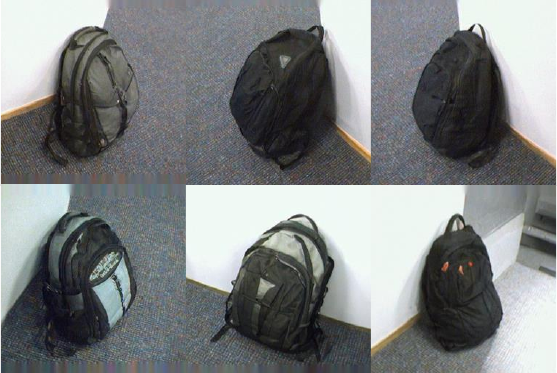}
    \label{fig:showcase1}
}
\subfigure[ELS+DACDM.]{
    \includegraphics[width=0.23\textwidth]{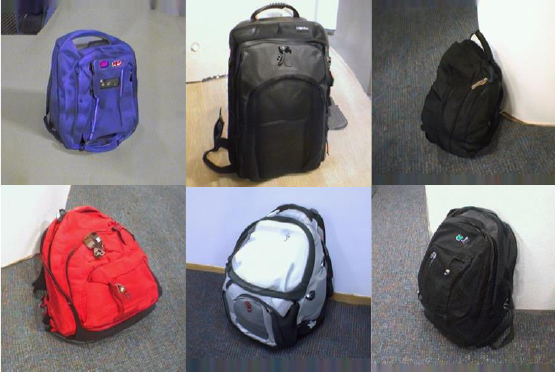}
    \label{fig:showcase2}
}\!\!
\vskip -.1in
\caption{Generated target images for the transfer task A$\rightarrow$W on the \textit{Office-31} dataset for different methods.}
\label{fig:showcase}
\end{figure}

\begin{figure}[!tbh]
\centering
\subfigure[]{
    \includegraphics[width=0.115\textwidth]{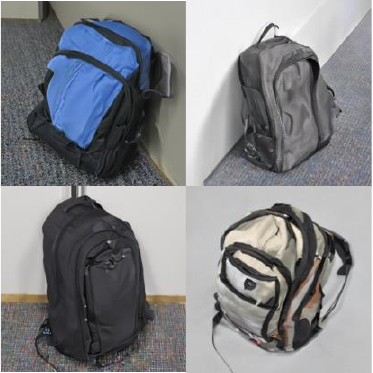}
    \label{fig:office31_ablation1}
}\hspace{-3mm}
\subfigure[]{
    \includegraphics[width=0.115\textwidth]{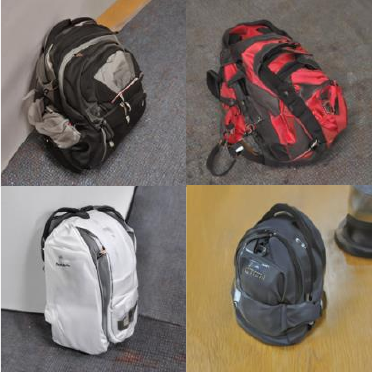}
    \label{fig:office31_ablation2}
}\hspace{-3mm}
\subfigure[]{
    \includegraphics[width=0.115\textwidth]{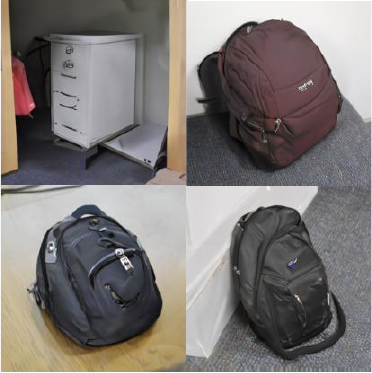}
    \label{fig:office31_ablation3}
}\hspace{-3mm}
\subfigure[]{
    \includegraphics[width=0.11\textwidth]{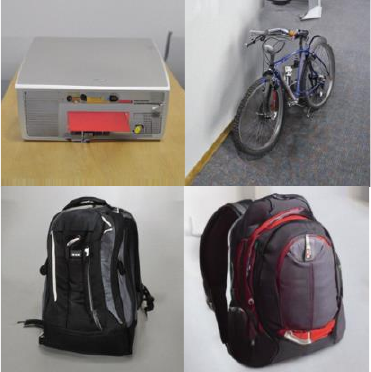}
    \label{fig:office31_ablation4}
}
\vskip -.1in
\caption{Generated target images for the transfer task A$\rightarrow$D on the \textit{Office-31} dataset for ablation studies. (a) The class and domain are controlled both by condition. (b) The classes are controlled by condition and domains are controlled by guidance. (c) The classes are controlled by guidance and domains are controlled by condition. (d) The class and domain are controlled both by guidance.}
\label{fig:office31_ablation}
\end{figure}

\subsection{Analysis on Convergence}
\label{Convergence}

Fig.~\ref{ELS_acc} shows convergence of the test accuracy for ELS and ELS+DACDM
w.r.t. the number of iterations for the six transfer tasks on \textit{Office-31}.
As can be seen, ELS+DACDM converges faster and generalizes better than ELS.
Fig.~\ref{MCC_acc} shows convergence of the test accuracy of MCC and MCC+DACDM w.r.t. the number of iterations on \textit{Office-31}.
MCC+DACDM
%, which incorporates DACDM into MCC, 
again performs better than MCC.
These demonstrate usefulness of DACDM.

\subsection{Ablation Studies}
\label{Ablation Studies}

\textbf{Effects of Components in DACDM.}
We perform ablation experiments on \textit{Office-31} to analyze the effects of 
(i) domain guidance, (ii) transferring from the augmented source domain to the target domain,
and (iii) using generated target samples.
We compare three variants of ELS+DACDM:
(i) 
DACDM trained on $\hD_t$ only (denoted by ELS+DACDM (w/o domain guidance));
(ii) $\hD_g$$\rightarrow$$\hD_t$ instead of $\hD_{\hat{s}}$$\rightarrow$$\hD_t$
% ELS trained on $\hD_g$ only 
(denoted by ELS+DACDM ($\hD_g$$\rightarrow$$\hD_t$));
(iii) ELS trained on $\hD_s$ and $\hD_t$ with pseudo labels, i.e., generated target samples are replaced by target samples with pseudo labels (denoted by ELS+PseudoLabel).

\begin{table}[!t]\small
\centering
\caption{Ablation studies about class and domain control on the \textit{Office-31} dataset using \textit{ResNet-50} in terms of accuracy (\%). The best is \textbf{in bold}.}
\vskip -.1in
\label{control}
\setlength{\tabcolsep}{1.5mm}{
\resizebox{\columnwidth}{!}{
\begin{tabular}{cccccccc @{\hskip 0.1in} c}
\toprule 
Class & Domain & A$\rightarrow$W  & D$\rightarrow$W  & W$\rightarrow$D  & A$\rightarrow$D   & D$\rightarrow$A & W$\rightarrow$A & Average \\
\midrule
Condition & Condition & 96.86 & \textbf{98.99} & \textbf{100.00} & 96.59 & 79.23 & 77.32 & 91.50 \\
Condition  & Guidance & \textbf{96.90} & 98.91 & \textbf{100.00} & \textbf{97.46} & \textbf{79.79} & \textbf{77.74} & \textbf{91.80} \\
Guidance  & Condition & 93.96 & 98.36 & \textbf{100.00} & 93.17 & 72.88 & 73.34 & 88.62 \\
Guidance  & Guidance & 90.82 & 96.73 & 99.00 & 90.56 & 70.86 & 68.65& 86.10 \\
\bottomrule
\end{tabular}}}
\end{table}

\begin{table}[!t]\small
\centering
\caption{Accuracy (\%) w.r.t. the number of generated samples per class on the \textit{Office-31} dataset.
The best is \textbf{in bold}.}
\vskip -.1in
\label{office31_number}
\setlength{\tabcolsep}{1.5mm}{
\resizebox{\columnwidth}{!}{
\begin{tabular}{ccccccc @{\hskip 0.2in} c}
\toprule 
& A$\rightarrow$W        & D$\rightarrow$W        & W$\rightarrow$D        & A$\rightarrow$D        & D$\rightarrow$A        & W$\rightarrow$A        & Average \\
\midrule
0  & 93.84 & 98.78 & \textbf{100.00} & 95.78 & 77.72 & 75.13 & 90.21 \\
50  & 97.15 & 98.74 & \textbf{100.00} & 96.92 & 79.54 & 76.86 & 91.54 \\
100  & \textbf{97.40} & 98.74 & \textbf{100.00} & 97.06 & 79.57 & 77.31 & 91.68 \\
200  & 96.90 & \textbf{98.91} & \textbf{100.00} & \textbf{97.46} & 79.79 & \textbf{77.74} & 91.80 \\
300  & \textbf{97.40} & 98.74 & \textbf{100.00} & 97.39 & 79.65 & 77.30 & 91.75 \\
400  & 97.15 & 98.74 & \textbf{100.00} & 97.26 & 80.06 & 77.46 & 91.78 \\
800  & 97.36 & 98.74 & 99.80 & 97.39 & \textbf{80.23} & \textbf{77.74} & \textbf{91.88} \\
1600  & 97.11 & 98.74 & 99.80 & 97.19 & 79.87 & 77.53 & 91.71 \\
\bottomrule
\end{tabular}}}
\end{table}

Table \ref{ablation} 
shows the results of six transfer tasks on \textit{Office-31}.
As can be seen,
ELS+DACDM
achieves better performance 
than ELS+DACDM (w/o domain guidance) on average,
demonstrating the effectiveness of using the domain guidance for generating samples.
Fig.~\ref{fig:showcase} shows the target backpack images generated by ELS+DACDM (w/o domain guidance) and ELS+DACDM for the transfer task A$\rightarrow$W. 
As can be seen, when training with only the target domain (w/o domain guidance),
only black backpacks can be generated (Fig.~\ref{fig:showcase1}).
In contrast, with the introduction of source samples and domain guidance,
backpacks of various colors following the target domain distribution are generated
(Fig.~\ref{fig:showcase2}).
Moreover, simply transferring from the generated target samples
to the target domain (i.e., ELS+DACDM ($\hD_g$$\rightarrow$$\hD_t$))
is slightly worse than ELS+DACDM on all six tasks, showing the advantage of transferring from the augmented source domain $\hD_{\hat{s}}$ to the target domain $\hD_t$.
Moreover, ELS+DACDM always performs better than ELS+PesudoLabel, verifying that using generated target samples is effective.

\textbf{Effects of Condition and Guidance in DACDM.}
We perform ablation experiments on \textit{Office-31} to further study the effects of different combinations of guidance and conditional methods for controlling classes and domains. \textcolor{blue}{The other parts of DACDM remain unchanged.}
There are four options:
(i)
Both classes and domains are controlled by condition,
where label embeddings and domain embeddings are added together
to train the conditional diffusion model in Algorithm \ref{algorithm:training_cdpm};
(ii) 
The classes are controlled by condition (Section~\ref{sec:condition}) while the domains are controlled by guidance (Section~\ref{sec:guidance}), i.e., the proposed DACDM;
(iii)
The classes are controlled by guidance while domains are controlled by condition,
where the classifier is trained on the source domain data and 
target domain samples with pseudo labels, and the domain labels are used as the condition $\vc$ in \eqref{eqn:loss_func_2};
(iv)
Both classes and domains are controlled by guidance, i.e., the sum of gradients of class classifier and domain classifier, which are trained separately.

Table \ref{control} shows the testing accuracy of six transfer tasks on \textit{Office-31}.
As can be seen, the design choice of using class condition and domain guidance in DACDM achieves the best performance, 
while controlling both classes and domains by condition is slightly inferior.
Furthermore,
when guidance is used for controlling classes (i.e., ``Guidance + Condition'' and ``Guidance + Guidance''),
the performance is largely worse.
This is because the trained class classifier may not be good enough for guiding, as some classes have limited samples, e.g.,  
in Dslr and Webcam domains, 16 and 26 samples per class in average;
however, the domain classification is a binary classification, which is much easier, thus, the trained domain classifier can be good enough for guidance. 
Figs.~\ref{fig:office31_ablation3} and \ref{fig:office31_ablation4} visualize some generated images when using ``Guidance + Condition'' and ``Guidance + Guidance'' (the controlled class is backpack).
As shown, some images are not backpack, leading to 
worse performance (Table \ref{control}).

\begin{table*}[!tbph]\small
\centering
\caption{Accuracy (\%) w.r.t. the number of generated samples per class on the \textit{Office-Home} dataset.
The best is in \textbf{bold}.}
\vskip -.1in
\label{officehome_number}
\setlength{\tabcolsep}{1mm}{
\begin{tabular}{ccccccccccccc @{\hskip 0.05in} c}
\toprule & Ar$\rightarrow$Cl & Ar$\rightarrow$Pr & Ar$\rightarrow$Rw & Cl$\rightarrow$Ar & Cl$\rightarrow$Pr & Cl$\rightarrow$Rw & Pr$\rightarrow$Ar & Pr$\rightarrow$Cl & Pr$\rightarrow$Rw & Rw$\rightarrow$Ar & Rw$\rightarrow$Cl & Rw$\rightarrow$Pr & Average \\
\midrule
0 & 57.79 & 77.65 & 81.62 & 66.59 & 76.74 & 76.43 & 62.69 & 56.69 & 82.12 &  75.63 & 62.85 & 85.35 & 71.84  \\
50 & 58.38 & 76.86 & 81.83 & 68.63 & 79.75 & 78.62 & 64.40 & 57.70 & 82.65 & 74.63 & 63.47 & 84.87 & 72.65 \\
100 & 58.38 & 77.89 & 81.74 & 68.41 & 79.87 & 78.40 & 64.15& 57.81 & 82.61 & 75.24 & 64.60 & 85.40 & 72.88  \\
200 & 60.35 & 78.81 & 82.74 & 69.59 & 80.53 & 79.55 & \textbf{65.16} & 58.26 & 83.11 & 75.81 & 64.18 & 85.55 & 73.64    \\
300 & 60.44 & 78.82 & 82.52 & 69.63 & 81.05 & 79.56 & 64.49 & 58.06 & 83.11 & 75.77 & 65.33 & 85.50 & 73.69 \\
400 & \textbf{61.14} & 78.82 & 82.93 & 69.47 & 81.21 & 79.88 & 64.51 & 58.96 & 83.24 & 75.28 & 65.54 & 85.75 & 73.89 \\
800 & 60.92 & 79.16 & 83.04 & 69.72 & 81.08 & 80.24 & 64.73 & \textbf{59.61} & 83.41 & 75.69 & \textbf{65.91} & \textbf{86.06} & 74.13 \\
1600  & 60.57 & 79.25 & 83.06 & \textbf{70.09} & \textbf{81.89} & \textbf{80.45} & 64.89 & 59.38 & \textbf{83.50} & 75.73 & 65.80 & 85.90 & \textbf{74.21} \\
2400  & 60.94 & \textbf{79.27} & \textbf{83.34} & 69.67 & 81.53 & 80.40 & 65.06 & 58.97 & 83.45 & \textbf{75.90} & 65.61 & 85.99 & 74.18 \\
\bottomrule
\end{tabular}}
\end{table*}

\begin{table*}[!tbph]\small
\centering
\caption{Accuracy (\%) w.r.t. the number of generated samples per class on the \textit{VisDA-2017} dataset. The best is in \textbf{bold}.}
\vskip -.1in
\label{visda_number}
\setlength{\tabcolsep}{1.5mm}{
\begin{tabular}{ccccccccccccc @{\hskip 0.2in} c}
\toprule  & aero & bicycle & bus & car & horse & knife & motor & person & plant & skate & train & truck & mean \\
\midrule
0 & 94.76 & 83.38 & 75.44 & 66.45 & 93.16 & 95.14 & 89.09 & 80.13 & 90.77 & 91.06 & 84.09 & 57.36 & 83.40 \\
500  & 96.16 & 84.29 & 79.29 & 69.91 & 94.20 & \textbf{96.79} & 89.35 & 71.72 & 92.93 & 92.49 & 86.55 & 57.05 & 84.23 \\
1000 & 95.68 & 85.32 & 81.22 & 72.36 & 93.74 & 96.58 & 90.43 & 81.91 & 93.02 & 92.97 & 86.34 & 56.42 & 85.50 \\
2000 & 96.20  & 84.79 & 83.15 & 73.28 & 94.76 & 96.58 & \textbf{90.99} & 82.21 & 92.98 & 93.37 & 87.49 & 59.70 & 86.29 \\
3000 & 96.31 & 86.13 & 81.50 & 72.28 & 94.50 & 96.74 & 90.25 & \textbf{82.92} & 93.21 & 93.13 & 88.09 & 60.30 & 86.28 \\
4000 & 96.19 & 86.23 & 83.08 & \textbf{74.39} & 94.24 & 96.66 & 89.54 & 82.73 & \textbf{94.87} & 94.01 & 87.72 & 60.11 & 86.65 \\
8000 & 96.77 & \textbf{88.49} & \textbf{85.06} & 72.48 & 95.35 & 96.21 & 90.33 & 82.49 & 94.52 & \textbf{94.52} & 88.32 & \textbf{60.89} & \textbf{87.15} \\
16000 & \textbf{97.01} & 87.48 & 83.45 & 72.59 & \textbf{95.76} & 95.37 & 89.91 & 82.78 & 94.04 & 94.39 & \textbf{89.94} & 59.77 & 86.87 \\
\bottomrule
\end{tabular}}
\end{table*}

\begin{table*}[!tbph]\small
\centering
\caption{Accuracy (\%) w.r.t. the number of generated samples per class on the \textit{miniDomainNet} dataset. The best is in \textbf{bold}.}
\vskip -.1in
\label{minidomainnet_number}
\setlength{\tabcolsep}{1.5mm}{
\begin{tabular}{ccccccccccccc @{\hskip 0.2in} c}
\toprule  & C$\rightarrow$P   & C$\rightarrow$R  & C$\rightarrow$S & P$\rightarrow$C & P$\rightarrow$R & P$\rightarrow$S & R$\rightarrow$C & R$\rightarrow$P & R$\rightarrow$S & S$\rightarrow$C & S$\rightarrow$P & S$\rightarrow$R & Average \\
\midrule
0  & 50.11 & 61.45 & 53.02 & 60.77 & 70.61 & 56.04 & 62.43 & 64.16 & 54.89 &  63.93 & 59.19 & 64.47 & 60.09 \\
50  & 54.12 & 66.20 & 54.78 & 62.88 & 72.73 & 56.28 & 63.45 & 64.13 & 56.19 & 65.60 & 60.12 & 67.15 & 61.97  \\
100  & 54.83 & 67.48 & 55.24 & 64.02 & 73.12 & 57.44 & 63.92 & 64.31 & 56.69 & 66.25 & 61.02 & 68.17 & 62.71 \\
200  & 56.14 & 68.52 & \textbf{55.51} & 64.96 & 73.44 & 58.23 & 64.67 & \textbf{64.36} & 57.66 & 67.23 & 62.56 & 69.90 & 63.60 \\
300  & \textbf{56.73} & 69.15 & 55.38 & 65.80 & 73.55 & \textbf{58.41} & 64.72 & 64.19 & 58.29 & 67.23 & 63.05 & 70.61 & 63.92 \\
400  & 56.57 & 69.48 & 55.25 & \textbf{66.42} & \textbf{73.77} & 58.18 & 65.09 & 64.09 & 58.02 & \textbf{67.46} & \textbf{63.31} & 71.06 & 64.06 \\
800  & 56.23 & 69.63 & 53.88 & 65.94 & 73.74 & 58.32 & \textbf{65.17} & 63.77 & \textbf{61.55} & 67.07 & 62.83 & 71.32 & \textbf{64.12} \\
1600  & 54.77 & \textbf{69.70} & 52.43 & 65.02 & 73.58 & 55.78 & 64.79 & 62.89 & 59.74 & 66.51 & 62.50 & \textbf{71.42} & 63.26 \\
\bottomrule
\end{tabular}}
\end{table*}

\subsection{Effects of the Number of Generated Samples}
\label{sec:sen_analysis}

In this section, we perform
experiments to study the effect of the number of generated target samples to the performance of ELS+DACDM on \textit{Office-31}, \textit{Office-Home}, \textit{VisDA-2017}, and \textit{miniDomainNet}.
Fig.~\ref{Sensitivity} shows the accuracy w.r.t. number of samples per class.
As can be seen, increasing the number of generated samples boosts performance, demonstrating effectiveness of the proposed DACDM. 
When there is no generated sample, ELS+DACDM degenerates to ELS, which performs worse than ELS+DACDM on all four datasets.
Furthermore, when the number of generated samples increases to a certain extent, 
the performance tends to stabilize and may even slightly decrease.
As a result,
the number of generated target samples for each class can be set to 200, 200, 2000, and 200 for \textit{Office-31}, \textit{Office-Home}, \textit{VisDA}, and \textit{miniDomainNet}, respectively, which is the default adopted in the above experiments.

Tables \ref{office31_number}, \ref{officehome_number}, \ref{visda_number}, and \ref{minidomainnet_number} show the detailed results of the transfer tasks on \textit{Office-31}, \textit{Office-Home}, \textit{VisDA-2017}, and \textit{miniDomainNet} datasets, respectively.
As can be seen, increasing the number of generated samples improves transfer performance, verifying the effectiveness of DACDM. 
For datasets with a small number of real samples per class (e.g.,
\textit{Office-31}, \textit{Office-Home}, and \textit{miniDomainNet}), a small number of generated samples per class (e.g., 200-400 per class) is sufficient for better performance, while for datasets with a larger number of real samples per class (e.g.,  \textit{VisDA-2017}), a large number of generated samples per class (e.g., 2,000-4,000 per class) is more preferred.

\section{Conclusion}
In this paper, we propose the DACDM method to generate high-quality samples for the target domain in UDA.
The classes of generated samples are controlled
by conditional diffusion models, while the domain is guided by the domain classifier.
DACDM can be integrated into any UDA model.
Extensive experimental results demonstrate 
that DACDM achieves state-of-the-art performance.
In the future work, we are interested in applying DACDM to other transfer learning settings such as multi-source domain adaptation \cite{sun2015survey}.% and domain generalization \cite{zhou2022domain}.

\section*{Acknowledgements}

This work is supported by NSFC key grant 62136005, NSFC general grant 62076118, and Shenzhen fundamental research program JCYJ20210324105000003.

\bibliographystyle{IEEEtran}
\bibliography{DACDM}

\end{document}